%% file: main.tex
\documentclass[twoside]{article}

\usepackage[accepted]{aistats2025}
\usepackage[utf8]{inputenc} 
\usepackage[T1]{fontenc}    
\usepackage{hyperref}       
\usepackage{url}            
\usepackage{booktabs}       
\usepackage{amsfonts}       
\usepackage{nicefrac}       
\usepackage{microtype}      
\usepackage{xcolor}         
\usepackage{mymacro}
\usepackage{tikz}
\usepackage[round]{natbib}

\begin{document}

%
\runningtitle{Optimal Private Estimation of Variance and Covariance}

%

\twocolumn[

  \aistatstitle{Optimal Variance and Covariance Estimation under Differential Privacy in the Add-Remove Model and Beyond}

  \aistatsauthor{ Shokichi Takakura \And Seng Pei Liew \And  Satoshi Hasegawa }

  \aistatsaddress{ LY Corporation \And  LY Corporation \And LY Corporation } ]

\begin{abstract}
  In this paper, we study the problem of estimating the variance and covariance of datasets under differential privacy in the add-remove model.
  While estimation in the swap model has been extensively studied in the literature, the add-remove model remains less explored and more challenging, as the dataset size must also be kept private.
  To address this issue, we develop efficient mechanisms for variance and covariance estimation based on the \emph{Bézier mechanism}, a novel moment-release framework that leverages Bernstein bases.
  We prove that our proposed mechanisms are minimax optimal in the high-privacy regime by establishing new minimax lower bounds.
  Moreover, beyond worst-case scenarios, we analyze instance-wise utility and show that the Bézier-based estimator consistently achieves better utility compared to alternative mechanisms.
  Finally, we demonstrate the effectiveness of the Bézier mechanism beyond variance and covariance estimation, showcasing its applicability to other statistical tasks.
\end{abstract}

\input{sections/body.tex}

\bibliography{main}

\input{sections/appendix.tex}

\end{document}

%% file: sections/body.tex
\section{Introduction}
Variance and covariance are fundamental statistics in data analysis.
They are widely used in various applications, such as machine learning, statistics, and data mining.
However, releasing statistics naively can lead to privacy issues,
as they can reveal sensitive information about individuals in the dataset, for example, through difference attacks~\citep{dalenius1977towards}.

Differential privacy (DP)~\citep{dwork2014algorithmic} is a widely used framework for protecting the privacy of individuals in a dataset.
It provides a formal guarantee that the output of a computation does not reveal too much information about any individual in the dataset.
In DP, there are two common settings for defining neighboring relations between datasets: \textit{swap model} and \textit{add-remove model}.
Intuitively, in the swap model, two datasets are neighboring if one can be obtained from the other by changing one record,
and in the add-remove model, two datasets are neighboring if one can be obtained from the other by adding or removing one record.
Since the add-remove model is more conservative in the sense that it protects the size of the dataset,
it is often preferred in practice~\citep{mcsherry2009privacy,wilson2020differentially,rogers2021linkedin,amin2022plume}.

Despite its practical importance, theoretical understanding of the add-remove model is less explored compared to the swap model.
Recently, \citet{kulesza2024mean} studied the problem of mean estimation in the add-remove model and provided a min-max lower bound and an optimal mechanism.
Interestingly, they showed that the mean estimation in the add-remove model is not harder than in the swap model
although the add-remove model treats the size of the dataset as a sensitive information.
However, estimation of variance and covariance is more complicated since it involves estimating the second moment of the dataset as well as the first moment.
Indeed, naive approaches suffer from large error compared to mean estimation.

Thus, we pose the following question:
\emph{Is variance and covariance estimation in the add-remove model harder than in the swap model or for mean estimation?}

In this paper, we answer this question negatively
by developing efficient mechanisms for variance and covariance estimation which achieve the same optimal error rate in the high-privacy regime
as in the swap model and mean estimation.
Specifically, we establish a novel lower bound for variance and covariance estimation
and propose min-max optimal mechanisms which achieve this lower bound.
A key ingredient of these optimal mechanisms is the \textit{B\'{e}zier mechanism}, a novel and versatile
moment-release framework based on Bernstein bases.
We show that the B\'{e}zier mechanism can estimate all moments up to order $k > 0$ simultaneously, without incurring any additional privacy cost compared to estimating only the $k$-th moment.
Furthermore, beyond variance and covariance estimation, we show that the B\'{e}zier mechanism can be applied to the estimation of a wide range of statistics, which is of independent interest.

Our contributions can be summarized as follows:
\begin{itemize}
    \item We establish new minimax lower bounds for variance and covariance estimation under differential privacy in both the swap and add-remove models.
          In the high-privacy regime, the optimal MSE is shown to be $\frac{2}{\varepsilon^2 n^2}$, matching the optimal MSE of mean estimation.
    \item We propose min-max optimal mechanisms for variance and covariance estimation in the add-remove model based on the \emph{B\'{e}zier mechanism},
          a novel moment-release framework based on Bernstein bases.
          Moreover, for variance estimation, we prove our B\'{e}zier-based estimator uniformly improves the leading constant compared to two alternative optimal mechanisms.
    \item We show that our framework can be applied to higher-order moments estimation,
          enabling efficient estimation of a wide class of moment-based quantities (e.g., kurtosis, skewness, and correlation).
\end{itemize}
\subsection{Other Related Work}
\paragraph{Covariance Estimation}
Covariance estimation with DP has been intensively studied in various settings~\citep{Amin2019differentially,biswas2020coinpress,dong2022differentially,kalinin2025continual}.
However, they focus on the swap model and studied the dependence of the error on the data dimension.
Furthermore, they considered the estimation of uncentered covariance matrix.
In contrast, we consider the estimation of \textit{centered} covariance in the add-remove model and focus on the constant factor in the error rate.

\paragraph{Other Statistics Estimation}
Estimation of other statistics other than variance and covariance such as
quantile~\citep{gillenwater2021differentially,imola2025differentially},
CDF~\citep{rameshwar2025optimal},
F statistics~\citep{swanberg2019improved},
and confidence interval~\citep{karwa2018finite}
is also the field of active research.
Several works~\citep{gillenwater2021differentially,imola2025differentially} have discussed the estimation in the add-remove model
but the optimal mechanism for variance and covariance estimation remains unexplored.

\subsection{Notation}
$\clip(\cdot, [l, u]):\R \to \R$ is the clipping function defined as $\clip(x, [l, u]) = \max\{l, \min\{u, x\}\}$.
$\Lap(b)$ is the Laplace distribution with scale parameter $b$, i.e., the probability density function of $\Lap(b)$ is given by $\frac{1}{2b} \exp(-\abs{x}/b)$.
With a slight abuse of notation, we denote $\Lap(b)$ as a random variable following the Laplace distribution with scale parameter $b$.
Small $o$-notation $o(f(n))$ is used to denote a function that grows slower than $f(n)$, i.e., $\lim_{n \to \infty} o(f(n))/f(n) = 0$.
We denote $[a_1, \dots, a_d] \in \R^{n\times d}$ for $a_i \in \R^n$ as the matrix whose $i$-th column is $a_i$.
\section{Problem Setting}
In this paper, we consider datasets containing records, each of which is a $d$-dimensional vector in a bounded domain $[l, r]^d$ $(l < r)$.
The case where the range of the records is not bounded can be transformed to the bounded case using the clipping technique~\citep{amin2019bounding}.
Without loss of generality, we can assume $l = 0$ and $r = 1$
since any dataset can be transformed to this form by a linear transformation.
Let $\mathcal{D}_n$ be the set of all datasets of size $n \geq 1$,
$\mathcal{D}_{\geq n} := \cup_{k \geq n} \mathcal{D}_k$ be the set of all datasets of size at least $n$,
and $\mathcal{D}_* := \cup_{k \geq 1} \mathcal{D}_k$ be the set of all datasets with any size.
We identify the dataset $D \in \mathcal{D}_n$ with the data matrix $[x^{(1)}, \dots, x^{(d)}] \in [0, 1]^{n \times d}$.
In this paper, we consider the problem of estimating the statistics such as variance and covariance defined as
\begin{align*}
    \vvar(D) & = \frac{1}{n}\sum_{i=1}^n x_i^2 - \left(\frac{1}{n}\sum_{i=1}^n x_i\right)^2,                                         \\
    \cov(D)  & = \frac{1}{n}\sum_{i=1}^n x_i y_i - \left(\frac{1}{n}\sum_{i=1}^n x_i\right)\left(\frac{1}{n}\sum_{i=1}^n y_i\right),
\end{align*}
where $D = [x] \in \mathcal{D}^*$, and $D = [x, y] \in \mathcal{D}^*$, respectively.
We emphasize that our focus is on centered covariance estimation,
whereas most previous work~\citep{Amin2019differentially,biswas2020coinpress,dong2022differentially,kalinin2025continual}
has studied the estimation of uncentered covariance, $\frac{1}{n}\sum_{i=1}^n x_i y_i$.

Here, we introduce the formal definition of differential privacy.
\begin{definition}[Differential Privacy]
    A randomized algorithm $\mathcal{A}$ is said to be $\epsilon$-differentially private if for all neighboring datasets $D, D' \in \mathcal{D}^*$ and all measurable sets $S \subseteq \text{Range}(\mathcal{A})$, it holds that
    \begin{align*}
        \Pr[\mathcal{A}(D) \in S] \leq e^{\epsilon} \Pr[\mathcal{A}(D') \in S].
    \end{align*}
\end{definition}
For the definition of \textit{neighboring} datasets, there are two common models:
\begin{itemize}
    \item \textbf{Swap Model}: The neighboring datasets are defined by replacing an element in the dataset with another one.
          Formally, two datasets $D, D' \in \mathcal{D}^*$ are neighboring if and only if $\abs{D \backslash D'} = 1$ and $\abs{D' \backslash D} = 1$.
    \item \textbf{Add-remove Model}: The neighboring datasets are defined by adding or removing an element from the dataset.
          Formally, two datasets $D, D' \in \mathcal{D}^*$ are neighboring if and only if $\abs{D \backslash D'} + \abs{D' \backslash D} = 1$.
\end{itemize}
In this paper, we focus on the add-remove model, which is preferred in practice since it protects the size of the dataset while the swap model does not
~\citep{mcsherry2009privacy,wilson2020differentially,rogers2021linkedin,amin2022plume}.

For a statistics $f:\mathcal{D}^* \to \R$, the utility of an estimator $\hat f: \mathcal{D}^* \to \R$ for a dataset $D \in \mathcal{D}^*$
is measured by the mean squared error defined as
\begin{align*}
    L(\hat f, f, D) & := \Expec{(\hat f(D) - f(D))^2},
\end{align*}
where the expectation is taken over the randomness of the estimator $\hat f$.
Since the error of differentially private estimator typically decreases in the order of $1/\abs{D}^2$~\citep{kulesza2024mean}, we define the (worst-case) normalized mean squared error as
\begin{align*}
    R(\hat f, f, D)   & = \abs{D}^2 L(\hat f, f, D),                           \\
    R(\hat f, f, n_0) & = \sup_{D \in \mathcal{D}_{\geq n_0}} R(\hat f, f, D).
\end{align*}
Our goal is to design an $\epsilon$-DP estimator $\hat f$ that minimizes the above measures.

\section{Lower Bound for Variance and Covariance Estimation}\label{sec:lower-bound}
In this section, we provide novel min-max lower bounds for variance and covariance estimation.
There have been several studies on the lower bound for covariance estimation~\citep{dong2022differentially,portella2024lower}
but they focused on the swap model and the uncentered covariance estimation.
Here, we show a novel lower bound for the centered covariance estimation in both the swap and add-remove models.
\begin{theorem}[Lower bound]\label{thm:lower-bound}
    For $\epsilon > 0$, any $\epsilon$-DP estimators $\hat{f}_\vvar, \hat{f}_\cov:\mathcal{D}^* \to \R$ in the swap or add-remove model
    satisfies
    \begin{align*}
        R(\hat f_{\cov}, \cov, n_0)   & \geq \sigma(\varepsilon)^2 (1 - o(1)), \\
        R(\hat f_{\vvar}, \vvar, n_0) & \geq \sigma(\varepsilon)^2 (1 - o(1)),
    \end{align*}
    where $\sigma(\varepsilon) = \frac{2^{-2/3}e^{-2\varepsilon / 3}(1 + e^{-\varepsilon})^{2/3} + e^{-\varepsilon}}{(1-e^{-\varepsilon})^2}$.
    In particular, if $\varepsilon \ll 1$, we have $\sigma(\varepsilon)^2 \simeq 2/\varepsilon^2$.
\end{theorem}
See Appendix~\ref{app:proof-lower-bound} for the proof.
This can be shown by using a lower-bound for the sum estimation~\citep{geng2014optimal} and a novel reduction of the variance and covariance estimation problem to the sum estimation problem.
In the high privacy regime, i.e., $\varepsilon \to 0$, the lower bound is asymptotically $2/\varepsilon^2$.
In this paper, we focus on the high-privacy regime since it is more challenging but provides meaningful privacy protection.
As shown later, this lower bound can be achieved by the Laplace mechanism in the swap model and by our proposed mechanism in the add-remove model.
That is, the leading term in the optimal mean squared error for variance and covariance estimation in high privacy regime is $2/(\varepsilon^2n^2)$ for large $n$, where $n$ is the size of the dataset.
Interestingly, this error rate has the same constant in the leading term as that of mean estimation~\citep{kulesza2024mean}.
Therefore, we can conclude that the variance and covariance estimation in the add-remove model is not harder than that in the swap model and the mean estimation.
\section{Optimal Mechanism for Swap Model}\label{sec:swap-model}
Estimation of variance and covariance in the swap model is rather straightforward
since the number of records in the dataset is fixed.
In this section, we briefly review the optimal mechanism for variance and covariance estimation in the swap model
to compare the result with the add-remove model.

Here, we introduce the notion of $l_1$-sensitivity.
\begin{definition}
    The $l_1$-sensitivity of a function $f:\mathcal{D}^* \to \R^d$ is defined as
    \begin{align*}
        \Delta f & = \sup_{D, D' \in \mathcal{D}^*} \norm{f(D) - f(D')}_1   \\
                 & \quad \quad \text{ s.t. } D, D' \text{ are neighboring}.
    \end{align*}
\end{definition}
The following fact is well known in the literature (see, e.g.,~\citet{Aitsam_2022}).
\begin{lemma}\label{lem:swap-model-sensitivity}
    The sensitivities of the empirical variance $\vvar(D)$ and covariance $\cov(D)$ are bounded by $1/n$.
\end{lemma}
We provide the proof in Appendix~\ref{proof:swap-model-sensitivity} for completeness.
Considering the above fact, we can use the Laplace mechanism to privatize the empirical variance and covariance as follows:
\begin{align*}
    \hat v_{\swap}(D) & = \vvar(D) + \frac{1}{n} \cdot \Lap(1/\varepsilon), \\
    \hat c_{\swap}(D) & = \cov(D) + \frac{1}{n} \cdot \Lap(1/\varepsilon).
\end{align*}
These mechanisms are $\varepsilon$-DP from Lemma~\ref{lem:swap-model-sensitivity}.
In addition, they achieve the lower bound in Theorem~\ref{thm:lower-bound}.
\begin{proposition}[Optimality in Swap Model]\label{prop:swap-utility}
    The Laplace mechanisms achieve the following error bound:
    \begin{align*}
        R(\hat v_{\swap}, \vvar, n_0) & = \frac{2}{\varepsilon^2}, \\
        R(\hat c_{\swap}, \cov, n_0)  & = \frac{2}{\varepsilon^2}.
    \end{align*}
\end{proposition}
The results follow from the fact that the variance of $\Lap(1/\varepsilon)$ is $2/\varepsilon^2$.
Therefore, for sufficiently small $\varepsilon$, we can conclude that the Laplace mechanism is optimal for variance and covariance estimation in the swap model.
In the low privacy regime $\varepsilon \gg 1$, there exists a line of work on improving the utility
using staircase mechanisms~\citep{geng2014optimal,kulesza2025general}
but we do not seek in this direction.

\section{Optimal Mechanism for Add-Remove Model}\label{sec:add-remove-model}
The situation in the add-remove model is more complicated than in the swap model
since the number of records in the dataset can change and should be protected.
We cannot apply the Laplace mechanism directly to the empirical variance and covariance
since the sensitivity is dependent on the size of the dataset, which is not fixed in the add-remove model.

A naive approach for the covariance estimation in the add-remove model
is to use the Laplace mechanism to privatize $n$, $\sum_{i=1}^n x_i$, $\sum_{i=1}^n y_i$, and $\sum_{i=1}^n x_i y_i$ separately.
That is, we use the privacy budget $\varepsilon / 4$ to privatize each of these quantities,
and obtain $\tilde n = n + \Lap(4/\varepsilon)$,
$\tilde s_x = \sum_{i=1}^n x_i + \Lap(4/\varepsilon)$,
$\tilde s_y = \sum_{i=1}^n y_i + \Lap(4/\varepsilon)$,
and $\tilde s_{xy} = \sum_{i=1}^n x_i y_i + \Lap(4/\varepsilon)$.
Then, we compute the empirical covariance as
\begin{align*}
    \hat c(D) := \clip\ab(\frac{\tilde s_{xy}}{\tilde n} - \frac{\tilde s_x}{\tilde n} \cdot \frac{\tilde s_y}{\tilde n}, [-1/4, 1/4]).
\end{align*}
Here, we apply clipping to improve the utility since the true empirical covariance lies in the range $[-1/4, 1/4]$.
Similarly, we can estimate the variance as
\begin{align*}
    \hat v(D) := \clip\ab(\frac{\tilde s_{xx}}{\tilde n} - \ab(\frac{\tilde s_x}{\tilde n})^2, [0, 1/4]),
\end{align*}
where $\tilde n = n + \Lap(3/\varepsilon)$, $\tilde s_x = \sum x_i + \Lap(3/\varepsilon)$, and $\tilde s_{xx} = \sum x_i^2 + \Lap(3/\varepsilon)$.
For these naive approaches, we have the following utility bound:
\begin{proposition}\label{prop:naive-utility}
    For $\hat c$ and $\hat v$ defined above, we have
    \begin{align*}
        R(\hat c, \cov, n_0)  & \leq \frac{128}{\varepsilon^2}(1 + o(1)), \\
        R(\hat v, \vvar, n_0) & \leq \frac{108}{\varepsilon^2}(1 + o(1))  \\
    \end{align*}
\end{proposition}
See Appendix~\ref{proof:naive-utility} for the proof.
This does not matches the lower bound in Theorem~\ref{thm:lower-bound} and is not optimal.

An improved approach utilizes the following property of unnormalized variance and covariance:
\begin{lemma}\label{lem:add-remove-model-sensitivity}
    The sensitivity of unnormalized covariance $\ucov([x, y]) := \sum_{i=1}^n (x_i - \frac{1}{n}\sum_{j=1}^n x_j)(y_i - \frac{1}{n}\sum_{j=1}^n y_j)$ and unnormalized variance $\uvar([x]) := \sum_{i=1}^n (x_i - \frac{1}{n}\sum_{j=1}^n x_j)^2$
    is bounded by $1$ in the add-remove model.
\end{lemma}
See Appendix~\ref{proof:add-remove-model-sensitivity} for the proof.
Let $u = \ucov([x, y])$ for covariance estimation and $u = \uvar([x])$ for variance estimation.
Then, the $l_1$-sensitivity of a vector $[n, u]$ is bounded by $2$,
and the Laplace mechanism with scale parameter $2/\varepsilon$ can be applied to privatize these quantities.
That is, we can obtain $\tilde n = n + \Lap(2/\varepsilon)$, $\tilde u = u + \Lap(2/\varepsilon)$.
Then, the empirical covariance and variance can be computed as
\begin{align*}
    \hat c(D) & := \clip\ab(\frac{\tilde u}{\tilde n}, [-1/4, 1/4]), \\
    \hat v(D) & := \clip\ab(\frac{\tilde u}{\tilde n}, [0, 1/4]).
\end{align*}
For these improved approaches, we have the following utility bound:
\begin{proposition}\label{prop:improved-utility}
    For $\hat c$ and $\hat v$ defined above, we have
    \begin{align*}
        R(\hat c, \cov, n_0)  & \leq \frac{17 / 2}{\varepsilon^2}(1 + o(1)), \\
        R(\hat v, \vvar, n_0) & \leq \frac{17 / 2}{\varepsilon^2}(1 + o(1)).
    \end{align*}
\end{proposition}
See Appendix~\ref{proof:improved-utility} for the proof.
While this approach reduces the error compared to the naive method, it still suffers from large error
and does not achieve the lower bound in Theorem~\ref{thm:lower-bound}.

\subsection{B\'{e}zier Mechanism}
To develop a mechanism which achieves the lower bound in Theorem~\ref{thm:lower-bound},
in this section, we propose a novel framework called \textit{B\'{e}zier Mechanism}.
First, we introduce the Bernstein basis and B\'{e}zier matrix~\citep{lorentz2012bernstein}.
\begin{definition}[Bernstein Basis and B\'{e}zier Matrix]
    The Bernstein basis of order $k \in \N$ is defined as
    \begin{align*}
        B_{j}(x) & = \binom{k}{j} x^j (1 - x)^{k - j}, \quad 0 \leq j \leq k.
    \end{align*}
    The B\'{e}zier matrix $A$ is defined as
    \begin{align*}
        A_{j, l} & = \begin{cases}
                         (-1)^{l - j} \binom{k}{l} \binom{l}{j} & \quad (j \leq l) \\
                         0                                      & \quad (j > l)
                     \end{cases}
    \end{align*}
\end{definition}
The Bernstein basis is widely used in computer graphics
and known to form a basis for polynomials of degree at most $k$.
In the context of DP,~\citet{alda2017bernstein} have proposed Bernstein mechanism based on Bernstein polynomial for private function release.
They utilize the property of Bernstein basis that it can approximate uniformly any continuous function.
On the other hand, we utilize the following property of Bernstein basis known as a \textit{partition of unity}:
\begin{itemize}
    \item $B_j(x)$ is non-negative.
    \item $\sum_{j=0}^k B_{j}(x) = 1$ for all $x \in [0, 1]$.
\end{itemize}

Let $\mu(D) \in \R^{k+1}$ be the vector of (unnormalized) moments of the dataset $D = [x]$ up to order $k$,
defined as
\begin{align*}
    \mu(D) = \left( \sum_{i=1}^n x_i^0, \sum_{i=1}^n x_i^1, \ldots, \sum_{i=1}^n x_i^k \right)^\top,
\end{align*}
and $b(D) \in \R^{k+1}$ be the Bernstein representation of $\mu$ defined as
\begin{align*}
    b(D) & = A\mu(D)                                                                                \\
         & = \ab[\sum_{i=1}^n B_0(x_i), \sum_{i=1}^n B_1(x_i), \ldots, \sum_{i=1}^n B_k(x_i)]^\top.
\end{align*}
From the partition of unity property of Bernstein basis, we have the following lemma:
\begin{lemma}
    The $l_1$-sensitivity of $b(D)$ is $1$.
\end{lemma}
\begin{proof}
    For neighboring datasets $D$ and $D'$, we have
    \begin{align*}
        \norm{b(D') - b(D)}_1 & = \sum_{j=0}^k \abs{B_j(x_{n+1})} \\
                              & = \sum_{j=0}^k B_j(x_{n+1})       \\
                              & = 1,
    \end{align*}
    where the first equality follows from the non-negativity of Bernstein basis and the last equality follows from the partition of unity property.
\end{proof}
This lemma implies that $\tilde b(D) = b(D) + \Lap(1 /\varepsilon)$ satisfy $\varepsilon$-DP.
Since the B\'{e}zier matrix $A$ is invertible, we can compute the noisy moments $\tilde \mu(D)$ from $\tilde b(D)$ as $\tilde \mu(D) = A^{-1}\tilde b(D)$.
We provide the entire procedure in Algorithm~\ref{alg:B\'{e}zier-mechanism}.
This can be seen as a special case of the matrix mechanism~\citep{li2015matrix} for the input is moments and the matrix is the B\'{e}zier matrix.
In addition, B\'{e}zier mechanism can be seen as a generalization of the transformed noise addition~\citet{kulesza2024mean}, which is a special case when $k=1$.

\begin{algorithm}[h]
    \caption{B\'{e}zier Mechanism}
    \label{alg:B\'{e}zier-mechanism}
    \begin{algorithmic}[1]
        \REQUIRE Dataset $[x] \in \mathcal{D}^*$, privacy budget $\varepsilon$, degree $k$.
        \STATE Let $b_j = \sum_{i=1}^{\abs{x}} B^k_j(x_i)$ for $j=0, \dots, k$.
        \STATE Add independent Laplace noise $Z_i\sim \Lap(1/\varepsilon)$ to obtain $\tilde b_i = b_i + Z_i$ for $i = 0, \dots, k$
        \STATE Output $\hat \mu(D) = [\hat \mu_0, \dots, \hat \mu_k]^\top = A^{-1} \tilde b$.
    \end{algorithmic}
\end{algorithm}

For the utility of Algorithm~\ref{alg:B\'{e}zier-mechanism}, we have the following result:
\begin{theorem}\label{thm:bezier-utility}
    For $0\leq j \leq k$, let $\mu_j$ be the $j$-th order moment of a dataset $D \in \mathcal{D}^*$.
    Then, the mechanism that outputs $\hat{\mu}_j$ satisfies
    \begin{align*}
        L(\hat \mu_j, \mu_j, D) = \frac{2}{\varepsilon^2 \cdot \binom{k}{j}}\sum_{l=j}^k\binom{l}{j}^2 \leq \frac{2k}{\varepsilon^2}
    \end{align*}
    In particular, $L(\hat \mu_k, \mu_k, D) = \frac{2}{\varepsilon^2}$, which matches the lower bound for $k$-th (unnormalized) moment estimation in the high-privacy regime.
\end{theorem}
See Appendix~\ref{app:proof-bezier-utility} for the proof.
For a naive mechanism which applies the Laplace mechanism to $\mu$, the error is $\frac{2(k+1)^2}{\varepsilon^2}$ since the $l_1$-sensitivity of $\mu$ is $k + 1$.
Thus, the B\'{e}zier mechanism achieves a better utility-privacy trade-off.
In addition, the B\'{e}zier mechanism achieves optimal error for the highest-order moment $\mu_k$.
Thus, we can estimate all moments up to order $k$ without paying additional privacy budget.

\subsection{Extension to Multivariate Cases}
For cases with two or more variables in a record, the extension is natural using tensor products of Bernstein basis.
For $d$ variables, the multivariate Bernstein basis is defined as:
\begin{align*}
    B_{\alpha}^{k}(z_1, \ldots, z_d) = \prod_{j=1}^{d} B_{\alpha_j}^{k}(z_j)
\end{align*}
where $z \in \R^d$ and $\alpha = (\alpha_1, \ldots, \alpha_d)$ is a multi-index with $\alpha_j \in \{0, 1, \ldots, k\}$.
The corresponding B\'{e}zier matrix becomes $A_d = A \otimes A \otimes \cdots \otimes A$ ($d$ tensor products).
In a similar way as in the univariate case, let $\mu(D) \in \R^{(k+1)^d}$ be the vector of mixed moments of
the dataset $D = [x^{(1)}, \ldots, x^{(d)}]$ up to order $k$,
defined as
\begin{align*}
    \mu_{\alpha}(D) = \frac{1}{n} \sum_{j=1}^{n} \ab(x^{(1)}_j)^{\alpha_1} \cdots \ab(x^{(d)}_j)^{\alpha_d},
\end{align*}
and $b$ be the Bernstein representation of $\mu$ defined as
\begin{align*}
    b_{\alpha}(D) = \sum_{i=1}^n B_{\alpha}(x_i).
\end{align*}
As in the univariate case, we have the following lemma:
\begin{lemma}\label{lem:multi-dim-sensitivity}
    The $l_1$-sensitivity of $b(D)$ is $1$.
\end{lemma}
See Appendix~\ref{proof:multi-dim-sensitivity} for the proof.
Thus, we can privatize $b(D)$ using the Laplace mechanism with scale parameter $1/\varepsilon$.
Since $A_d$ is invertible, we can recover all mixed moments up to order $k$ from Bernstein representation.

\subsection{Covariance Estimation}
To compute the covariance, it suffices to estimate $n, \sum_{i=1}^n x_i, \sum_{i=1}^n y_i, \sum_{i=1}^n x_iy_i$.
This can be done using the two-dimensional B\'{e}zier mechanism with degree $k=1$.
That is, we compute the noisy Bernstein representation
\begin{align*}
    \tilde b_{0, 0} & = \sum_{i=1}^{n} (1-x_i)(1-y_i)+\Lap(1/\varepsilon),    \\
    \tilde b_{0, 1} & = \sum_{i=1}^{n} (1-x_i)y_i(1-x_i)+\Lap(1/\varepsilon), \\
    \tilde b_{1, 0} & = \sum_{i=1}^{n} x_i(1-y_i)+\Lap(1/\varepsilon),        \\
    \tilde b_{1, 1} & = \sum_{i=1}^{n} x_i y_i+\Lap(1/\varepsilon)
\end{align*}
and then, reconstruct $\tilde s_{xy} = \tilde b_{1, 1}$, $\tilde s_x = \tilde b_{1, 0} + \tilde b_{1, 1}$, $\tilde s_y = \tilde b_{0, 1} + \tilde b_{1, 1}$, and $\tilde n = \sum_{i=0}^{1} \sum_{j=0}^{1} \tilde b_{i, j}$.
Finally, we compute the noisy covariance as $\hat c = \clip\ab(\frac{\tilde s_{xy}}{\tilde n} - \ab(\frac{\tilde s_x}{\tilde n})\ab(\frac{\tilde s_y}{\tilde n}), [-1/4, 1/4])$.
Note that $\hat c$ is $\varepsilon$-DP by the post-processing theorem.
We provide the entire procedure in Algorithm~\ref{alg:covariance-estimation-add-remove}.
\begin{algorithm}[h]
    \caption{Optimal Covariance Estimation in the Add-Remove Model}
    \label{alg:covariance-estimation-add-remove}
    \begin{algorithmic}[1]
        \REQUIRE Dataset $[x, y] \in \mathcal{D}^*$, privacy budget $\varepsilon$
        \STATE Let $b_{1, 1} = \sum_{i=1}^{\abs{x}} x_i y_i$
        \STATE Let $b_{1, 0} = \sum_{i=1}^{\abs{x}} x_i(1-y_i)$
        \STATE Let $b_{0, 1} = \sum_{i=1}^{\abs{y}} y_i(1-x_i)$
        \STATE Let $b_{0, 0} = \sum_{i=1}^{\abs{x}} (1-x_i)(1-y_i)$
        \STATE Add independent Laplace noise $Z_{i, j}\sim \Lap(1/\varepsilon)$ to obtain $\tilde b_{i, j} = b_{i, j} + Z_{i, j}$.
        \STATE Calculate $\tilde n = \sum_{i=0}^{1} \sum_{j=0}^{1} \tilde b_{i, j}$, $\tilde s_x = \tilde b_{1, 1} + \tilde b_{1, 0}$, $\tilde s_y = \tilde b_{0, 1} + \tilde b_{0, 0}$, and $\tilde s_{xy} = \tilde b_{1, 1}$
        \STATE Output $\hat c = \clip\ab(\frac{\tilde s_{xy}}{\tilde n} - \ab(\frac{\tilde s_x}{\tilde n})\ab(\frac{\tilde s_y}{\tilde n}), [-1/4, 1/4])$
    \end{algorithmic}
\end{algorithm}

The utility of Algorithm~\ref{alg:covariance-estimation-add-remove} is evaluated as follows:
\begin{theorem}[Utility Analysis]\label{thm:covariance-estimation-add-remove-utility}
    For privacy budget $\varepsilon$, Algorithm~\ref{alg:covariance-estimation-add-remove} achieves
    \begin{align*}
        R(\hat c, \cov, n_0) & \leq \frac{2}{\varepsilon^2}(1 + o(1)).
    \end{align*}
\end{theorem}
See Appendix~\ref{app:proof-covariance-estimation-add-remove-utility} for the proof.
While the B\'{e}zier mechanism can efficiently estimate the moments of the dataset, it is far from trivial to evaluate the error in the final estimator.
Surprisingly, this upper bound matches the lower bound in Theorem~\ref{thm:lower-bound} for small $\varepsilon$.
Therefore, Algorithm~\ref{alg:covariance-estimation-add-remove} is minimax optimal for covariance estimation in the high-privacy regime.
This is because the noise contained in each individual moment estimator is correlated with one another,
and many of them eventually cancel out in the final estimator.

\subsection{Variance Estimation}
The empirical variance can be computed from the moments $n, \sum_{i=1}^n x_i, \sum_{i=1}^n x_i^2$ up to order 2.
This can be done using the (univariate) B\'{e}zier mechanism with degree $k=2$.
Then,
\begin{align*}
    \hat v_b(D) & = \clip\ab(\frac{\tilde s_{x^2}}{\tilde n} - \ab(\frac{\tilde s_{x}}{\tilde n})^2, [0, 1/4]),
\end{align*}
where $\tilde s_{x^2} = \tilde b_2$, $\tilde s_x = \tilde b_1 / 2 + \tilde b_2$, and $\tilde n = \sum_{i=0}^{2} \tilde b_i$.
We provide the entire procedure in Algorithm~\ref{alg:variance-estimation-add-remove}.

\begin{algorithm}[h]
    \caption{Optimal Variance Estimation in the Add-Remove Model}
    \label{alg:variance-estimation-add-remove}
    \begin{algorithmic}[1]
        \REQUIRE Dataset $[x] \in \mathcal{D}^*$, privacy budget $\varepsilon$
        \STATE Let $b_0 = \sum_{i=1}^{\abs{x}} (1 - x_i)^2$
        \STATE Let $b_1 = \sum_{i=1}^{\abs{x}} 2x_i (1 - x_i)$
        \STATE Let $b_2 = \sum_{i=1}^{\abs{x}} x_i^2$
        \STATE Add independent Laplace noise $Z_i\sim \Lap(1/\varepsilon)$ to obtain $\tilde b_i = b_i + Z_i$ for $i = 0, 1, 2$
        \STATE Compute $\tilde n = \sum_{i=0}^{2} \tilde b_i$, $\tilde s_x = \tilde b_1 / 2 + \tilde b_2$, and $\tilde s_{x^2} = \tilde b_2$
        \STATE Output $\clip\ab(\frac{\tilde s_{x^2}}{\tilde n} - \ab(\frac{\tilde s_x}{\tilde n})^2, [0, 1/4])$
    \end{algorithmic}
\end{algorithm}

The utility of Algorithm~\ref{alg:variance-estimation-add-remove} is evaluated as follows:
\begin{theorem}[Utility analysis]\label{thm:variance-estimation-add-remove-utility}
    For privacy budget $\varepsilon$, Algorithm~\ref{alg:variance-estimation-add-remove} achieves
    \begin{align*}
        R(\hat v_b, n_0) & \leq \frac{2}{\varepsilon^2}(1 + o(1)).
    \end{align*}
\end{theorem}
See Appendix~\ref{app:proof-variance-estimation-add-remove-utility} for the proof.
Since this upper bound matches the lower bound in Theorem~\ref{thm:lower-bound} for small $\varepsilon$,
Algorithm~\ref{alg:variance-estimation-add-remove} is minimax optimal for variance estimation in the high privacy regime.

\subsection{Comparison of Optimal Mechanisms for Variance Estimation}\label{sec:comparison-variance}
In the previous section, we proposed the optimal mechanism for variance estimation.
In this section, we provide two alternative mechanisms which also achieve the lower-bound in Theorem~\ref{thm:lower-bound}.
While these mechanisms are optimal in the worst-case sense, we further derive instance-dependent error bounds that explicitly depend on the dataset’s mean and variance.
This analysis demonstrates that Algorithm~\ref{alg:variance-estimation-add-remove} always provides better utility for individual datasets.

The first alternative mechanism exploits the close relationship between variance and covariance:
variance estimation is a special case of covariance estimation where we set $x = y$.
Accordingly, we can apply Algorithm~\ref{alg:covariance-estimation-add-remove} to variance estimation.
That is, we can use $\hat v_{c}([x]) = \hat c([x, x])$ as an estimator for variance.
The utility bound for $\hat v_c$ can be obtained from Theorem~\ref{thm:covariance-estimation-add-remove-utility}.
\begin{corollary}
    For privacy budget $\varepsilon$, $\hat v_c$ achieves
    \begin{align*}
        R(\hat v_c, \vvar, n_0) & \leq \frac{2}{\varepsilon^2}(1 + o(1)).
    \end{align*}
\end{corollary}
Since this upper bound matches the lower bound of variance estimation in Theorem~\ref{thm:lower-bound},
we can conclude that Algorithm~\ref{alg:covariance-estimation-add-remove} is also minimax optimal for variance estimation in the add-remove model.

To construct another optimal mechanism for variance estimation,
we revisit the improved mechanism with unnormalized covariance.
Let $u = \ucov(x)$, and $b_0 = n - u, b_1 = u$.
Then, we have the following lemma:
\begin{lemma}\label{lem:transformed-sensitivity}
    The $l_1$-sensitivity of $b = [n - u, u]^\top$ is $1$
\end{lemma}
See Appendix~\ref{app:proof-transformed-sensitivity} for the proof.
Therefore, $\tilde b := b + \Lap(1/\varepsilon)$ is $\varepsilon$-DP
and noisy variance can be computed as $\hat v_u(D) := \clip\ab(\frac{\tilde u}{\tilde n}, [0, 1/4])$,
where $\tilde n = \tilde b_0 + \tilde b_1$ and $\tilde u = \tilde b_2$.
See Algorithm~\ref{alg:variance-estimation-transformed} for the entire procedure.
As shown in the following proposition, $\hat v_u$ also achieves the lower bound in Theorem~\ref{thm:lower-bound}.
\begin{proposition}\label{prop:variance-estimation-transformed-utility}
    For privacy budget $\varepsilon$, $\hat v_u$ achieves
    \begin{align*}
        R(\hat v_u, \vvar, n_0) & \leq \frac{2}{\varepsilon^2}(1 + o(1)).
    \end{align*}
\end{proposition}
See Appendix~\ref{app:proof-variance-estimation-transformed-utility} for the proof.
This mechanism can be seen as a generalization of transformed noise addition for mean estimation in~\citet{kulesza2024mean} to variance estimation.
Note that this approach cannot be applied to covariance estimation since the sensitivity space of $\ucov(D)$ is different from that of $\uvar(D)$,
and the sensitivity of $b = [n - \ucov([x, y]), \ucov([x, y])]$ is not bounded by $1$.

\begin{algorithm}[h]
    \caption{Optimal Variance Estimation in the Add-Remove Model}
    \label{alg:variance-estimation-transformed}
    \begin{algorithmic}[1]
        \REQUIRE Dataset $[x] \in \mathcal{D}^*$, privacy budget $\varepsilon$
        \STATE Let $b_0 = \sum_{i=1}^{\abs{x}} 1 - x_i^2 + \frac{1}{\abs{x}} (\sum_{i=1}^{\abs{x}} x_i)^2$
        \STATE Let $b_1 = \sum_{i=1}^{\abs{x}} x_i^2 - \frac{1}{\abs{x}} (\sum_{i=1}^{\abs{x}} x_i)^2$
        \STATE Add independent Laplace noise $Z_i\sim \Lap(1/\varepsilon)$ to obtain $\tilde b_i = b_i + Z_i$ for $i = 0, 1$
        \STATE Compute $\tilde n = \sum_{i=0}^{1} \tilde b_i$, $\tilde v_x = \tilde b_1$.
        \STATE Output $\clip\ab(\frac{\tilde v_x}{\tilde n}, [0, 1/4])$
    \end{algorithmic}
\end{algorithm}

While Algorithm~\ref{alg:covariance-estimation-add-remove},~\ref{alg:variance-estimation-add-remove}, and~\ref{alg:variance-estimation-transformed} achieve the same error rate in the worst case,
they may perform differently in practice due to their distinct mechanisms and the specific characteristics of the data.
Here, we compare the constants in the leading term of the instance-wise error $R(\cdot, \vvar, D)$
and prove that Algorithm~\ref{alg:variance-estimation-add-remove} is more efficient than the other two algorithms
for variance estimation.
\begin{theorem}\label{thm:comparison}
    For $r = \frac{1}{n} \sum_{i=1}^n x_i$ and $v = \frac{1}{n} \sum_{i=1}^n (x_i - r)^2$,
    let us define
    \begin{align*}
        C_b(r, v) & = 3v^2 - 2(3r^2 - 3r + 1)v           \\
                  & \quad + 3r^4 - 6r^3 + 7r^2 - 4r + 1, \\
        C_c(r, v) & = (1-2r+2r^2)^2 - 2v(1-2r)^2 + 4v^2, \\
        C_u(r, v) & = v^2 + (1-v)^2.
    \end{align*}
    Then, normalized errors $R(\cdot, f, D)$ for Algorithm~\ref{alg:covariance-estimation-add-remove},
    ~\ref{alg:variance-estimation-add-remove}, and~\ref{alg:variance-estimation-transformed}
    satisfies
    \begin{align*}
        R(\hat v_b, \vvar, n_0) & \leq \frac{2}{\varepsilon^2}(C_b(r, v) + o(1)), \\
        R(\hat v_c, \vvar, n_0) & \leq \frac{2}{\varepsilon^2}(C_c(r, v) + o(1)), \\
        R(\hat v_u, \vvar, n_0) & \leq \frac{2}{\varepsilon^2}(C_u(r, v) + o(1)).
    \end{align*}
    Furthermore, for any possible $r$ and $v$, we have
    \begin{align*}
        C_b(r, v) & \leq C_c(r, v) \leq C_u(r, v).
    \end{align*}
\end{theorem}
See Appendix~\ref{proof:comparison} for the proof.
Note that $r$ and $v$ are not independent and should satisfy certain constraints for the existence of the dataset.
Since $\frac{2}{\varepsilon^2}C_{\cdot}(r, v)$ is a first-order approximation of the error of each algorithm if we ignore the effect of clipping,
we can use it to analyze the performance of the algorithms.
We see that Algorithm~\ref{alg:variance-estimation-add-remove} consistently outperforms the other algorithms for datasets with any possible mean and variance.
We empirically confirm the superiority of Algorithm~\ref{alg:variance-estimation-add-remove} in Section~\ref{sec:experiments} with finite samples.

\section{Extension to General Statistics}
While our main focus is on the variance and covariance estimation, our proposed B\'{e}zier mechanism is versatile and can be used to estimate other statistics beyond variance and covariance.
In general, given a statistic $f(\mu(D))$ which depends on the (mixed) moments $\mu(D)$ up to order $k$,
we can use the B\'{e}zier mechanism with degree $k$ to estimate $\mu(D)$
and compute an $\varepsilon$-DP estimator $\hat f(D) = f(\text{B\'{e}zierMechanism}(D;\varepsilon, k))$
by the post-processing theorem.
This formulation includes higher-order centered moments, skewness, kurtosis, and correlation.
We provide a detailed procedure in Algorithm~\ref{alg:general-B\'{e}zier}.
\begin{algorithm}[h]
    \caption{B\'{e}zier Mechanism for General Statistics}
    \label{alg:general-B\'{e}zier}
    \begin{algorithmic}[1]
        \REQUIRE Dataset $D \in \mathcal{D}^*$, privacy budget $\varepsilon$, order $k$, range $[l, u]$.
        \STATE Let $b_\alpha(D) = \sum_{i=1}^n B_{\alpha_j}(x_{i, 1}) \cdot \dots \cdot \sum_{i=1}^n B_{\alpha_j}(x_{i, d})$ for $\alpha_i = 0, \dots, k$.
        \STATE Add independent Laplace noise $\tilde b_\alpha(D) = b_\alpha(D) + \Lap(1/\varepsilon)$.
        \STATE Compute $\tilde \mu(D) = (A_d)^{-1} \tilde b(D)$.
        \STATE Output $\hat f(D) := \clip\ab(f(\tilde \mu(D)), [l, u])$
    \end{algorithmic}
\end{algorithm}

\section{Numerical Experiments}\label{sec:experiments}
To see the empirical performance of our proposed mechanisms and compare it with naive approaches,
we conduct numerical experiments on synthetic datasets.
The results are averaged over 10,000 independent runs.
\begin{figure}
    \centering
    \begin{minipage}[t]{0.23\textwidth}
        \centering
        \includegraphics[width=\textwidth]{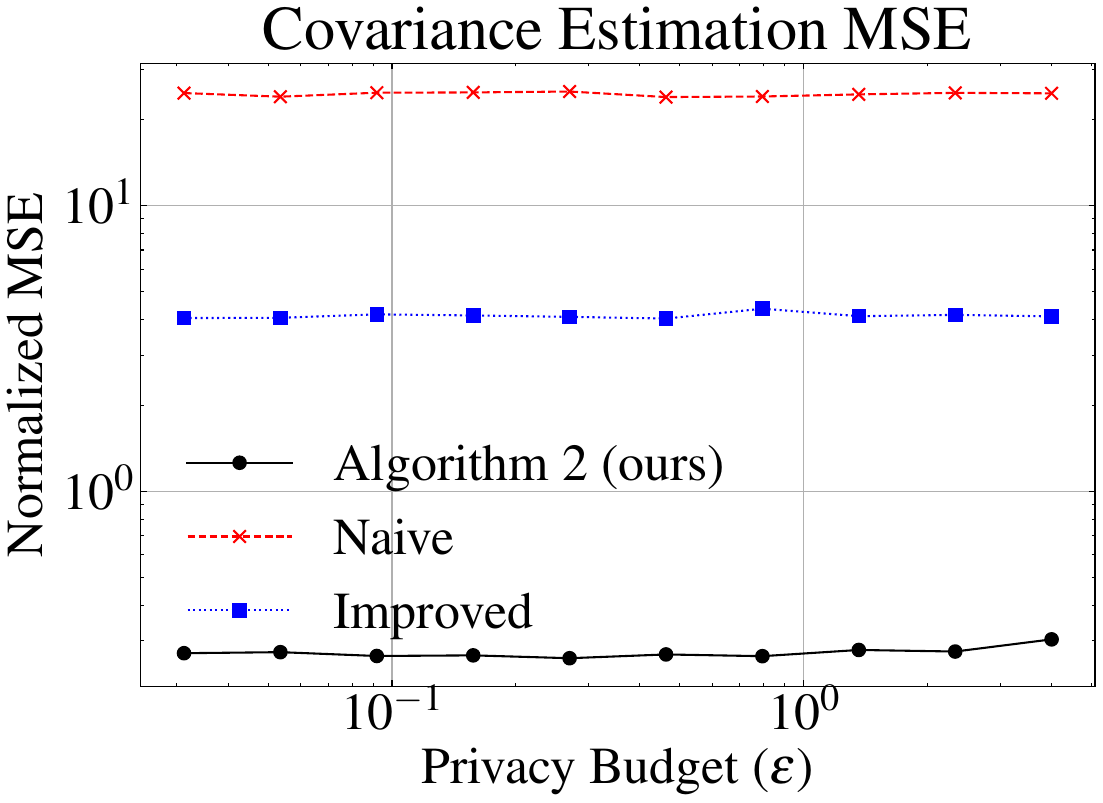}
    \end{minipage}
    \begin{minipage}[t]{0.23\textwidth}
        \centering
        \includegraphics[width=\textwidth]{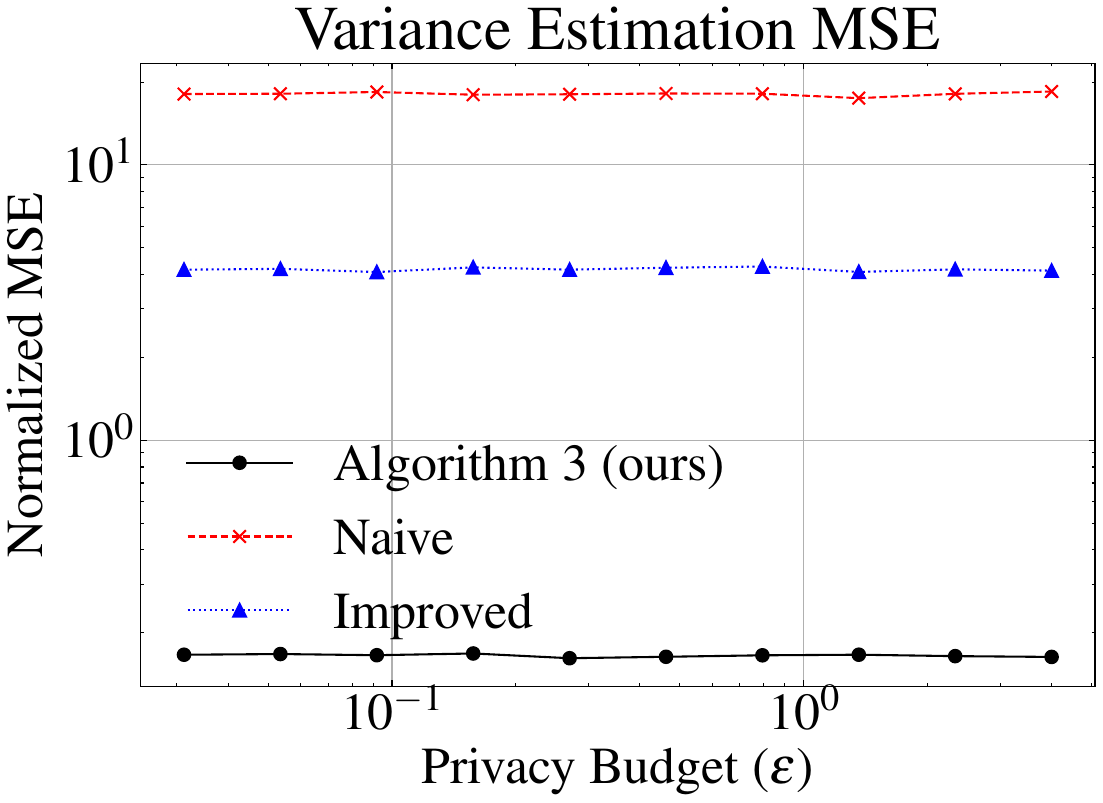}
    \end{minipage}
    \label{fig:comparison}
    \caption{Comparison of covariance and variance estimation. B\'{e}zier mechanism consistently outperforms baselines for different $\varepsilon$.}
\end{figure}
\begin{figure}[t]
    \centering
    \begin{minipage}[t]{0.23\textwidth}
        \centering
        \includegraphics[width=\textwidth]{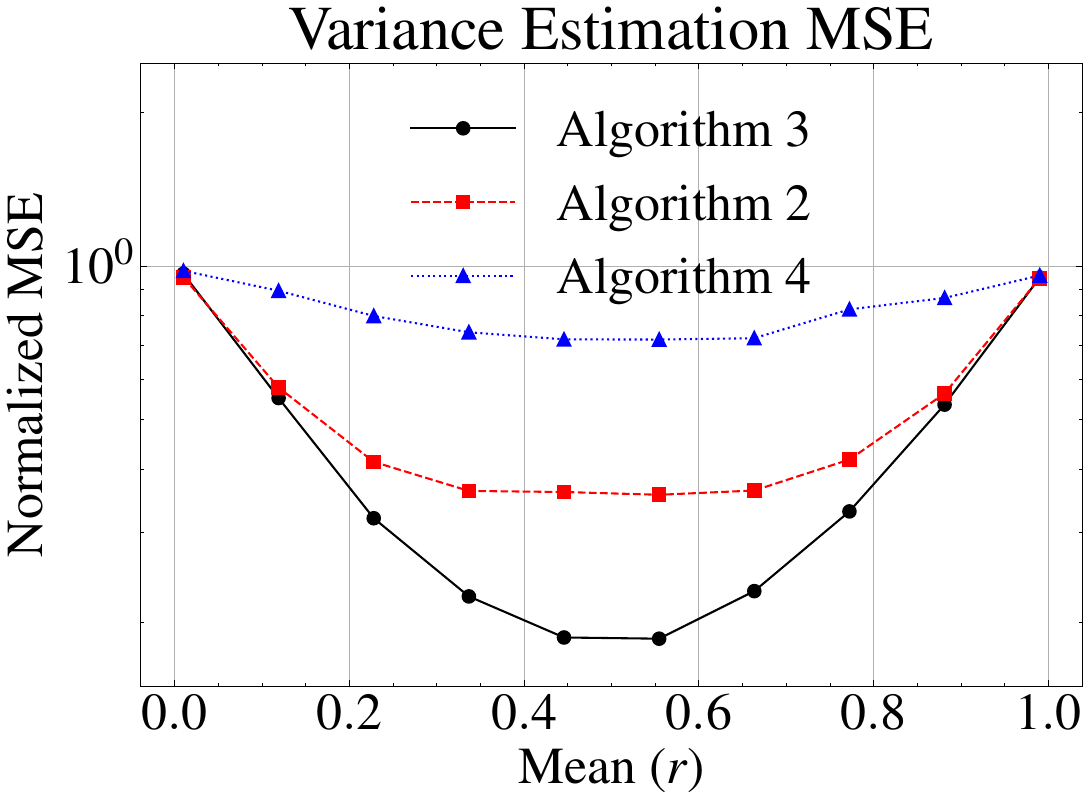}
    \end{minipage}
    \begin{minipage}[t]{0.23\textwidth}
        \centering
        \includegraphics[width=\textwidth]{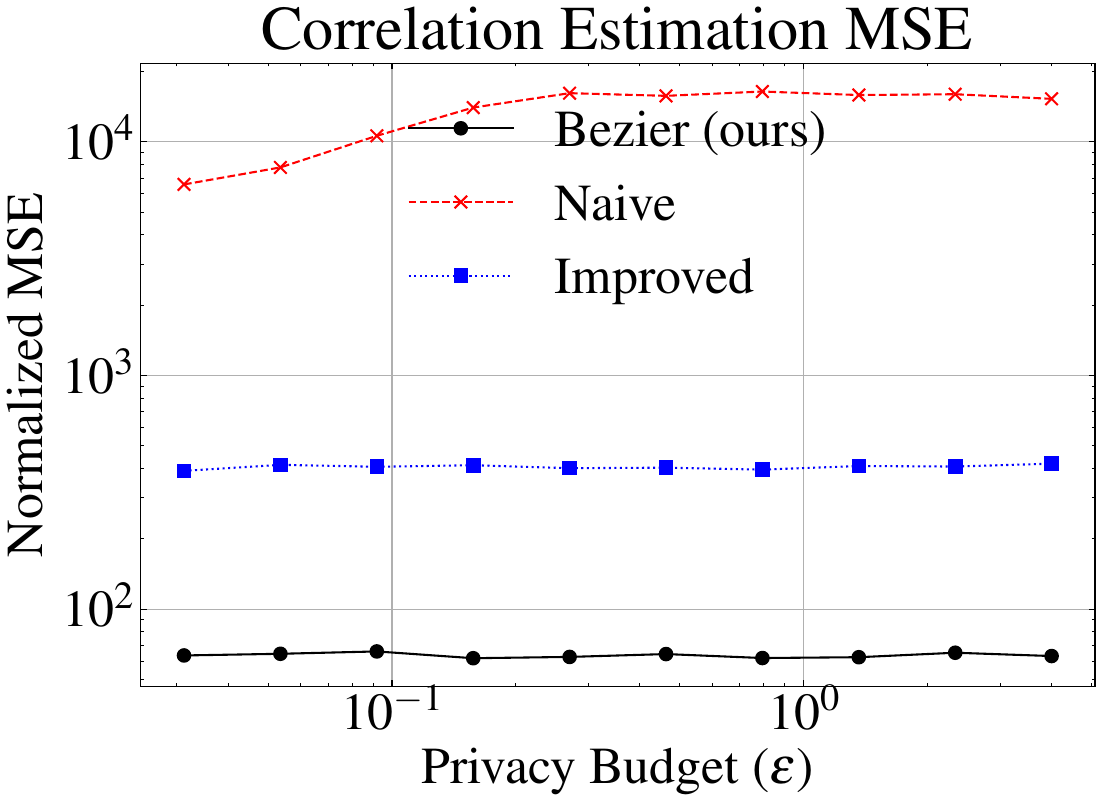}
    \end{minipage}
    \label{fig:comparison-variance-correlation}
    \caption{Comparison of optimal mechanisms for variance estimation (left) and correlation estimation (right). B\'{e}zier mechanism consistently outperforms baselines.}
\end{figure}
\paragraph{B\'{e}zier Mechanism Outperforms Baselines in Covariance and Variance Estimation}
We compare our proposed mechanism in Algorithm~\ref{alg:covariance-estimation-add-remove} and Algorithm~\ref{alg:variance-estimation-add-remove}
with baselines including 1) naive approach: applying the Laplace mechanism to $n$, $\sum_{i=1}^n x_i$, $\sum_{i=1}^n y_i$, and $\sum_{i=1}^n x_i y_i$ (or $n$, $\sum x_i$, $\sum x_i^2$) separately,
and 2) improved approach: applying Laplace mechanism to estimate $n$ and $\ucov(D)$ (or $\uvar(D)$) separately.
We generate a dataset of size $n = 10,000$ with $x_i, y_i$ sampled from uniform distribution on $[0, 1]$.
Fig.~\ref{fig:comparison} shows that our proposed mechanisms reduces the error to less than about 1/10 compared to baselines.

\paragraph{Comparison of Optimal Mechanisms for Variance Estimation}
We compare our proposed mechanisms for variance estimation in
Algorithm~\ref{alg:covariance-estimation-add-remove},
Algorithm~\ref{alg:variance-estimation-add-remove},
and Algorithm~\ref{alg:variance-estimation-transformed} with different mean $r$.
Note that we use Algorithm~\ref{alg:covariance-estimation-add-remove} as a mechanism for variance estimation, as described in Section~\ref{sec:comparison-variance}.
We generate a dataset of size $n = 10,000$ with $x_i$ sampled from beta distribution with parameters $(r / 2, (1 - r) / 2)$ to control the mean of datasets.
We fix the privacy budget $\varepsilon = 1$ and compute the error of the three mechanisms.
While both achieve the optimal error in the worst case, Fig.~\ref{fig:comparison-variance-correlation} (left) shows that Algorithm~\ref{alg:variance-estimation-add-remove}, based on B\'{e}zier mechanism, is more efficient than the other algorithms especially in the case $r \simeq 0.5$.

\paragraph{B\'{e}zier Mechanism is Effective Beyond Variance and Covariance}
As an example of moment-based statistics other than variance and covariance,
we consider the Pearson correlation coefficient defined as
\begin{align*}
    \corr(D) = \frac{\cov([x, y])}{\sqrt{\vvar([x]) \cdot \vvar([y])}}.
\end{align*}
This quantity can be computed with mixed moments up to the $k=2$ order.
Thus, we can use the two-dimensional B\'{e}zier mechanism with degree $k=2$ to estimate the correlation.
Figure~\ref{fig:comparison-variance-correlation} (right) shows the results of comparing with baselines including 1) naive approach: applying Laplace mechanism to privatize the moments independently,
and 2) improved approach: applying Algorithm~\ref{alg:covariance-estimation-add-remove} and~\ref{alg:variance-estimation-add-remove} to privatize $\cov([x, y])$ and $\vvar([x])$, $\vvar([y])$ separately.
We see that Algorithm~\ref{alg:general-B\'{e}zier} achieves a better utility-privacy trade-off with large margin.

\section{Conclusion}
In this paper, we study the problem of variance and covariance estimation with differential privacy in the add-remove model.
We propose efficient mechanisms for variance and covariance estimation based on a novel framework called the B\'{e}zier mechanism, and prove that our methods are minimax optimal in the high-privacy regime.
We further present two alternative optimal mechanisms and show that, beyond worst-case guarantees, the B\'{e}zier-based estimator achieves uniformly smaller leading constants for any data distribution.
Finally, we demonstrate that the B\'{e}zier mechanism is also effective for the estimation of other statistics, such as correlation.

%% file: sections/appendix.tex
\appendix
\onecolumn


\section{Auxiliary Results} \label{app:auxiliary-results}
\begin{lemma}
    For $x, y \in [0, 1]^n$, let $c := \frac{1}{n}\sum_{i=1}^n (x_i - \frac{1}{n}\sum_{i=1}^n x_i)(y_i - \frac{1}{n}\sum_{i=1}^n y_i)$
    and $v := \frac{1}{n}\sum_{i=1}^n (x_i - \frac{1}{n}\sum_{i=1}^n x_i)^2$.
    Then, we have
    \begin{align*}
        -1/4 \leq c \leq 1/4, \\
        0 \leq v \leq 1/4.
    \end{align*}
\end{lemma}
\begin{proof}
    From the Cauchy-Schwarz inequality, we have
    \begin{align*}
        c^2 \leq v^2.
    \end{align*}
    Thus, it suffices to show that $0 \leq v \leq 1/4$.
    From the definition of $v$, we have $0 \leq v$.
    Let $r = \frac{1}{n} \sum_{i=1}^n x_i$.
    Since $x_i^2 \leq x_i$ from $x_i \in [0, 1]$, we have
    \begin{align*}
        \frac{1}{n} \sum_{i=1}^n x_i^2 \leq r
    \end{align*}
    Therefore, we obtain
    \begin{align*}
        v & = \frac{1}{n}\sum_{i=1}^n x_i^2 - r^2 \\
          & \leq r - r^2 \leq 1/4
    \end{align*}
    since $r \in [0, 1]$.
\end{proof}
\begin{lemma}\label{lem:X+Y}
    Let $X, Y$ be random variables with bounded second moments.
    Then, we have
    \begin{align*}
        (1 - p)\Expec{X^2} + (1 - 1/p)\Expec{Y^2} \leq \Expec{(X+Y)^2} \leq (1 + p)\Expec{X^2} + (1 + 1/p)\Expec{Y^2}
    \end{align*}
    for any $p > 0$.
\end{lemma}
\begin{proof}
    For any $p>0$, we have
    \begin{align*}
        \Expec{(X+Y)^2} & = \Expec{X^2} + 2\Expec{XY} + \Expec{Y^2}                       \\
                        & \leq \Expec{X^2} + 2\sqrt{\Expec{X^2}\Expec{Y^2}} + \Expec{Y^2} \\
                        & \leq (1 + p)\Expec{X^2} + (1 + 1/p)\Expec{Y^2}.
    \end{align*}
    The first inequality follows from the Cauchy-Schwarz inequality, and the second inequality follows from the AM-GM inequality.
    On the other hand, we have
    \begin{align*}
        \Expec{(X+Y)^2} & = \Expec{X^2} + 2\Expec{XY} + \Expec{Y^2}                       \\
                        & \geq \Expec{X^2} - 2\sqrt{\Expec{X^2}\Expec{Y^2}} + \Expec{Y^2} \\
                        & \geq (1 - p)\Expec{X^2} + (1 - 1/p)\Expec{Y^2}.
    \end{align*}
    The first inequality follows from the Cauchy-Schwarz inequality, and the second inequality follows from the AM-GM inequality.
\end{proof}

\begin{lemma}\label{lem:rxy}
    For $x, y \in [0, 1]^n$, we have
    \begin{align*}
        \max(0, r_x + r_y - 1) \leq r_{xy} \leq \min(r_x, r_y),
    \end{align*}
    where $r_x = \frac{1}{n} \sum_{i=1}^{n} x_i$, $r_y = \frac{1}{n} \sum_{i=1}^{n} y_i$, and $r_{xy} = \frac{1}{n} \sum_{i=1}^{n} x_i y_i$.
\end{lemma}
\begin{proof}
    For the upper bound, we have
    \begin{align*}
        r_{xy} & = \frac{1}{n} \sum_{i=1}^{n} x_i y_i               \\
               & = \frac{1}{n} \sum_{i=1}^{n} x_i \cdot y_i         \\
               & \leq \frac{1}{n} \sum_{i=1}^{n} x_i \cdot 1 = r_x.
    \end{align*}
    Similarly, we can show that $r_{xy} \leq r_y$.
    For the lower bound, we have
    \begin{align*}
        \frac{1}{n} \sum_{i=1}^{n} (1 - x_i)(1-y_i) & = \frac{1}{n} \sum_{i=1}^{n} (1 - x_i - y_i + x_i y_i) \\
                                                    & = 1 - r_x - r_y + r_{xy} \geq 0,
    \end{align*}
    which implies that
    \begin{align*}
        r_{xy} \geq r_x + r_y - 1.
    \end{align*}
    Thus, we conclude that
    \begin{align*}
        \max(0, r_x + r_y - 1) \leq r_{xy} \leq \min(r_x, r_y).
    \end{align*}
\end{proof}
\begin{lemma}\label{lem:proxy-bound}
    For $f: \R^d \to \R$ and $A_n = \{\norm{Z}_\infty \leq n^{1/4}\} \in \R^d$, assume that
    $f$ is twice continuously differentiable on $\{Z / n \mid Z \in A_n\}$ and
    $\grad f(Z/n), \grad^2 f(Z / n) = O(1)$ for $Z \in A_n$,
    and a random variable $Z$ satisfies $\Prob{Z \in A_n} = 1 - o(1 / n^2)$ and $\Expec{\norm{Z}^2} = O(1)$.
    Then, we have
    \begin{align*}
        n^2 \cdot \Expec[Z]{\norm{\clip(f_n(Z / n), [l, u]) - f(0)}^2} \leq \Expec[Z]{\norm{\Delta(Z)}^2} + o(1)
    \end{align*}
    for any $l, u$ such that $f(0) \in [l, u]$,
    where $\Delta(Z) = \langle \grad f(0), Z\rangle$.
\end{lemma}
\begin{proof}
    From the intermediate value theorem, we have
    \begin{align*}
        f(Z / n) - f(0) & = \langle \grad f(0), Z/n \rangle + \frac{1}{2n^2} Z^\top \grad^2 f(\xi) Z \\
                        & = \langle \grad f(0), Z/n \rangle + O(\norm{Z}_2^2 / n^2)                  \\
                        & = \Delta(Z/n) + O(n^{-3/2}).
    \end{align*}
    for some $\xi$ on the line segment between $0$ and $Z / n$.
    Therefore, we obtain
    \begin{align*}
        n^2 \cdot \Expec{\norm{\clip(f(Z / n), [l, u]) - f(0)}^2} & = n^2 \cdot \Expec{\norm{\clip(f(Z / n), [l, u]) - f(0)}^2 \cdot 1_A}                                     \\
                                                                  & \quad + n^2 \cdot \Expec{\norm{\clip(f(Z / n), [l, u]) - f(0)}^2 \cdot (1 - 1_A)},                        \\
                                                                  & \leq n^2\cdot \Expec{\norm{f(Z / n) - f(0)}^2 \cdot 1_A} + n^2 \cdot \abs{u - l} \cdot \Prob{Z \notin A}, \\
                                                                  & \leq n^2\cdot \Expec{\norm{f(Z / n) - f(0)}^2 \cdot 1_A} + \abs{u - l} \cdot o(1).
    \end{align*}
    For the first inequality, we use
    The first term can be bounded as follows:
    \begin{align*}
        n^2 \cdot \Expec[Z \sim \nu_n]{\norm{f(Z) - f(0)}^2 \cdot 1_A} & = n^2 \cdot \Expec{\ab(\Delta(Z/n) + O(n^{-3/2}))^2 \cdot 1_A}           \\
                                                                       & \leq \Expec{\ab(\Delta(Z) + O(n^{-1/2}))^2}                              \\
                                                                       & \leq (1 + n^{-1/2}) \Expec{\norm{\Delta(Z)}^2} + (1 + n^{1/2}) O(n^{-1}) \\
                                                                       & \leq \Expec{\norm{\Delta(Z)}^2} + o(1).
    \end{align*}
    The second inequality follows from Lemma~\ref{lem:X+Y} and the last inequality follows from $\Expec{\norm{\Delta(Z)}^2} = O(1)$.
    This completes the proof.
\end{proof}

\section{Proofs for Section~\ref{sec:lower-bound}}
\subsection{Proof of Theorem~\ref{thm:lower-bound}}\label{app:proof-lower-bound}
\begin{proof}
    We prove the theorem by reducing the problem to the sum estimation problem.

    First, we consider the variance estimation.
    Given $\varepsilon$-DP mechanism $\hat v:\mathcal{D}^* \to \R$,
    we construct the following mechanism $\hat s_k:\mathcal{D}^* \to \R$ for some $k \in \N$:
    \begin{align*}
        \hat s_k(D) & = k\hat v(D_k),
    \end{align*}
    where $D = [x]$ and $D_k = [[\sqrt{x_1}, \sqrt{x_2}, \ldots, \sqrt{x_m}, \underbrace{0, \dots, 0}_{k}]^\top]$.
    We show that $\hat s_k$ is an $\epsilon$-DP estimator of the sum of the dataset $x$.
    For neighboring datasets $D, D' \in \mathcal{D}^*$ and $S \subset \R$, we have
    \begin{align*}
        \Prob{\hat s_k(D) \in S} & = \Prob{k\hat v(D_k) \in S}                                      \\
                                 & = \Prob{\hat v(D_k) \in \{s / k \mid s \in S\}}                  \\
                                 & \leq e^{\epsilon} \Prob{\hat v(D_k') \in \{s / k \mid s \in S\}} \\
                                 & = e^{\epsilon} \Prob{k\hat v(D_k') \in S}                        \\
                                 & = e^{\epsilon} \Prob{\hat s_k(D') \in S},
    \end{align*}
    where the inequality follows from the definition of differential privacy
    and the fact that $D_k$ and $D_k'$ are neighboring datasets if $D$ and $D'$ are neighboring datasets.

    From Lemma 5.1 in~\citet{kulesza2024mean}, we have
    \begin{align*}
        \sup_{D \in \mathcal{D^*}}\Expec{(\hat{s}_k(D) - \Sum(D))^2} \geq \sigma(\varepsilon)^2(1 - o(1)),
    \end{align*}
    where $\Sum(D)$ is the true sum of dataset $D$.
    Thus, there exists $D \in \mathcal{D}^*$ such that
    \begin{align}
        \Expec{(\hat{s}_k(D) - \Sum(D))^2} \geq \sigma(\varepsilon)^2(1 - o(1)). \label{eq:sum-lower-bound}
    \end{align}

    The error $\abs{k \cdot \vvar(D_k) - \Sum(D)}$ can be evaluated as follows:
    \begin{align*}
        \abs{k\cdot \vvar(D_k) - \Sum(D)}
         & = \abs{\frac{k}{\abs{D} + k}\sum_{i=1}^{\abs{D}} x_i - \frac{k}{(\abs{D} + k)^2} \ab(\sum_{i=1}^{\abs{D}} \sqrt{x_i})^2 - \sum_{i=1}^{\abs{D}} x_i} \\
         & \leq \frac{1}{k} \ab(\sum_{i=1}^{\abs{D}} \sqrt{x_i})^2 + \abs{\frac{\abs{D}}{\abs{D} + k} \sum_{i=1}^{\abs{D}} x_i}                                \\                                           \\
         & \leq \frac{2\abs{D}^2}{k}
    \end{align*}

    Therefore, we obtain
    \begin{align*}
        \Expec{\abs{D_k}^2 (\hat{v}(D_k) - \vvar(D_k))^2}
         & \geq \Expec{k^2 (\hat{v}(D_k) - \vvar(D_k))^2}                                                \\
         & = \Expec{(\hat{s}_k(D) - \Sum(D) + \Sum(D) - k\vvar(D_k))^2}                                  \\
         & \geq (1 - p)\Expec{(\hat{s}_k(D) - \Sum(D))^2} + (1 - 1 / p)\Expec{(\Sum(D) - k\vvar(D_k))^2} \\
         & \geq (1 - p)\Expec{(\hat{s}_k(D) - \Sum(D))^2} - \frac{4(1 - 1/p)\abs{D}^4}{k^2}              \\
         & \geq (1 - p)\sigma(\varepsilon)^2(1 - o(1)) - \frac{4(1 - 1/p)\abs{D}^4}{k^2}
    \end{align*}
    for any $p>0$.
    The second inequality follows from Lemma~\ref{lem:X+Y} and the last inequality follows from Eq.~\eqref{eq:sum-lower-bound}.
    By setting $k = \Omega(\abs{D}^4)$, $p = 1/k$, and $n_0 = k$, we have
    \begin{align*}
        R(\hat v, n_0) \geq \Expec{\abs{D_k}^2 (\hat{v}(D_k) - \vvar(D_k))^2}
         & \geq \sigma(\varepsilon)^2(1 - o(1)).
    \end{align*}

    Since variance estimation is a special case of covariance estimation where $D = [x, x]$,
    we can use the same argument to show the lower bound for covariance estimation.
    This completes the proof.
\end{proof}

\section{Proofs for Section~\ref{sec:swap-model}} \label{app:proof-swap-model}

\subsection{Proof for Lemma~\ref{lem:swap-model-sensitivity}} \label{proof:swap-model-sensitivity}
\subsubsection{Variance Case}
\begin{proof}
    Let $[x], [y]$ be neighboring datasets, i.e., $x_i = y_i$ for all $i \neq j$ and $x_j \neq y_j$.
    Without loss of generality, we assume $j = n$ and $x_n \geq y_n$.
    To simplify the notation, we denote $r = \frac{1}{n} \sum_{i=1}^{n-1}x_i$.
    Then, we have
    \begin{align*}
        \vvar(x) - \vvar(y) & = \frac{1}{n}\sum_{i=1}^n x_i^2 - \left(\frac{1}{n}\sum_{i=1}^n x_i\right)^2 - \ab(\frac{1}{n}\sum_{i=1}^n y_i^2 - \left(\frac{1}{n}\sum_{i=1}^n y_i\right)^2) \\
                            & = \frac{1}{n} (x_n^2 - y_n^2) - \ab((r + x_n / n)^2 - (r + y_n / n)^2)                                                                                         \\
                            & = \frac{1}{n} (x_n^2 - y_n^2) - (2rx_n / n + x_n^2 / n^2 - 2ry_n / n - y_n^2 / n^2)                                                                            \\
                            & = \ab(\frac{1}{n} - \frac{1}{n^2}) (x_n^2 - y_n^2) - \frac{2r}{n}(x_n - y_n).
    \end{align*}
    Since $r \in [0, (n-1)/n]$ and $x_n \geq y_n$, we obtain
    \begin{align*}
        -1/n \leq & - \frac{n-1}{n^2}                                                                  \\
                  & \leq \frac{n-1}{n^2}\ab((x_n - 1)^2 - (y_n - 1)^2)                                 \\
                  & = \ab(\frac{1}{n} - \frac{1}{n^2}) (x_n^2 - y_n^2) - \frac{2(n-1)}{n^2}(x_n - y_n) \\
                  & \leq \vvar(x) - \vvar(y)                                                           \\
                  & \leq \ab(\frac{1}{n} - \frac{1}{n^2})                                              \\
                  & \leq \frac{1}{n}.
    \end{align*}
    This completes the proof.
\end{proof}

\subsubsection{Covariance Case}
\begin{proof}
    Let $[x, y], [z, w]$ be neighboring datasets, i.e., $(x_i, y_i) = (z_i, w_i)$ for all $i \neq j$ and $(x_j, y_j) \neq (z_j, w_j)$.
    Without loss of generality, we assume $j = n$ and $x_n \geq z_n$.
    To simplify the notation, we denote $r_x = \frac{1}{n} \sum_{i=1}^{n-1}x_i$ and $r_y = \frac{1}{n} \sum_{i=1}^{n-1}y_i$.
    Then, we have
    \begin{align*}
        \cov(x, y) - \cov(z, w) & = \frac{1}{n}\sum_{i=1}^n x_i y_i - \left(\frac{1}{n}\sum_{i=1}^n x_i\right)\left(\frac{1}{n}\sum_{i=1}^n y_i\right) - \ab(\frac{1}{n}\sum_{i=1}^n z_i w_i - \left(\frac{1}{n}\sum_{i=1}^n z_i\right)\left(\frac{1}{n}\sum_{i=1}^n w_i\right)) \\
                                & = \frac{1}{n} (x_n y_n - z_n w_n) - \ab((r_x + x_n / n)(r_y + y_n / n) - (r_x + z_n / n)(r_y + w_n / n))                                                                                                                                       \\
                                & = \frac{1}{n} (x_n y_n - z_n w_n) - (r_x y_n / n + r_y x_n / n + x_n y_n / n^2 - r_x w_n / n - r_y z_n / n - z_n w_n / n^2)                                                                                                                    \\
                                & = \ab(\frac{1}{n} - \frac{1}{n^2}) (x_n y_n - z_n w_n) - \frac{1}{n}(r_x (y_n - w_n) + r_y (x_n - z_n)).
    \end{align*}
    If $y_n \geq w_n$, we have
    \begin{align*}
        -1/n \leq & - \frac{n-1}{n^2}                                                \\
                  & \leq \frac{n-1}{n^2}\ab((x_n - 1)(y_n - 1) - (z_n - 1)(w_n - 1)) \\
                  & \leq \cov(x, y) - \cov(z, w)                                     \\
                  & \leq \ab(\frac{1}{n} - \frac{1}{n^2}) (x_n y_n - z_n w_n)        \\
                  & \leq \ab(\frac{1}{n} - \frac{1}{n^2})                            \\
                  & \leq \frac{1}{n}.
    \end{align*}
    On the other hand, if $y_n < w_n$, we have
    \begin{align*}
        -1/n \leq & - \frac{n-1}{n^2}                                                   \\
                  & \leq \frac{n-1}{n^2}\ab(x_n(y_n - 1) - z_n(w_n - 1))                \\
                  & \leq \cov(x, y) - \cov(z, w)                                        \\
                  & \leq \ab(\frac{1}{n} - \frac{1}{n^2})\ab((x_n-1)y_n - (z_n - 1)w_n) \\
                  & \leq \ab(\frac{1}{n} - \frac{1}{n^2})                               \\
                  & \leq \frac{1}{n}.
    \end{align*}
    Combining both cases, we obtain the result.
\end{proof}

\section{Proofs for Section~\ref{sec:add-remove-model}}

\subsection{Proof for Proposition~\ref{prop:naive-utility}}\label{proof:naive-utility}
\subsubsection{Covariance Case}
\begin{proof}
    Let $Z = [Z_n, Z_x, Z_y, Z_{xy}]$ and $f(Z / n) = \frac{r_{xy} + Z_{xy}/n}{1 + Z_n / n} - \frac{r_{x} + Z_{x} / n}{1 + Z_n / n}\frac{r_y + Z_y / n}{1 + Z_n / n}$.
    Note that $f(0)$ is the true covariance and $\hat c(D) = \clip(f(Z / n), [-1/4, 1/4])$,
    where $Z \sim \Lap(4/\varepsilon)$.
    It is easy to see that $\norm{\grad f(Z / n)}, \norm{\grad^2 f(Z / n)} = O(1)$ for $Z \in A_n$.
    Then, from Lemma~\ref{lem:proxy-bound}, it suffices to show
    1) $\Prob{Z \in A_n} = 1 - o(1/n^2)$,
    2) $\Expec{\norm{\Delta(Z)}^2} \leq \frac{128}{\varepsilon^2}$.
    We prove these three conditions in the following lemmas:
    \begin{lemma}\label{lem:probability-bound}
        We have $\Prob{Z \in A_n} = 1 - o(1/n^2)$.
    \end{lemma}
    \begin{proof}
        Since $Z_n, Z_x, Z_y, Z_{xy} \sim \Lap(1/\varepsilon)$, we have
        \begin{align*}
            \Prob{Z \notin A_n} & \leq \Prob{\abs{Z_n} > n^{1/4}} + \Prob{\abs{Z_x} > n^{1/4}} + \Prob{\abs{Z_y} > n^{1/4}} + \Prob{\abs{Z_{xy}} > n^{1/4}} \\
                                & = O(\exp(-\varepsilon n^{1/4})) = o(1/n^2).
        \end{align*}
        This completes the proof.
    \end{proof}
    \begin{lemma}
        We have
        \begin{align*}
            \Expec{\norm{\Delta(Z)}^2} & \leq \frac{128}{\varepsilon^2}.
        \end{align*}
    \end{lemma}
    \begin{proof}
        From simple calculation, we have
        \begin{align*}
            \Delta(Z) = (-r_{xy} + 2r_xr_y)Z_n - r_y Z_x - r_x Z_y + Z_{xy}.
        \end{align*}
        From Lemma~\ref{lem:rxy}, we have
        \begin{align*}
            2r_xr_y - r_{xy} & \leq 2r_xr_y - \max(0, r_x + r_y - 1)       \\
                             & = r_xr_y + \min(r_xr_y, (1 - r_x)(1 - r_y)) \\
                             & \leq r_x r_y + (1 - r_x)(1 - r_y)           \\
                             & \leq 1,                                     \\
            2r_xr_y - r_{xy} & \geq 2r_xr_y - \min(r_x, r_y) \geq -1.
        \end{align*}
        Since $Z \sim \Lap(4/\varepsilon)$, we obtain
        \begin{align*}
            \Expec{\norm{\Delta(Z)}^2} & = 32/\varepsilon^2 \ab((-r_{xy} + 2r_x r_y)^2 + r_y^2 + r_x^2 + 1) \\
                                       & \leq 128 / \varepsilon^2.
        \end{align*}
        The equality holds when $r_x = r_y = r_{xy} = 1$.
    \end{proof}
\end{proof}
\subsubsection{Variance Case}
Let $Z = [Z_n, Z_{x}, Z_{xx}]$ and $f(Z / n) = \frac{r_{xx} + Z_{xx}/n}{1 + Z_n / n} - \ab(\frac{r_{x} + Z_{x} / n}{1 + Z_n / n})^2$.
As in the case of covariance, it suffices to show
1) $\Prob{Z \in A_n} = 1 - o(1/n^2)$,
2) $\Expec{\norm{\Delta(Z)}^2} = \frac{108}{\varepsilon^2}$.
The proof of 1) is the same as Lemma~\ref{lem:probability-bound}. Thus, we prove 2) in the following lemma:
\begin{lemma}
    We have
    \begin{align*}
        \Expec{\norm{\Delta(Z)}^2} & \leq \frac{108}{\varepsilon^2}.
    \end{align*}
\end{lemma}
\begin{proof}
    From simple algebra, we have
    \begin{align*}
        \Delta(Z) = (-r_{xx} + 2r_x^2)Z_n - 2r_x Z_x + Z_{xx}.
    \end{align*}
    Since $Z \sim \Lap(3/\varepsilon)$, we obtain
    \begin{align*}
        \Expec{\norm{\Delta(Z)}^2} & = 18/\varepsilon^2 \ab((-r_{xx} + 2r_x^2)^2 + 4r_x^2 + 1) \\
                                   & \leq 108 / \varepsilon^2.
    \end{align*}
    Here, the equality holds when $r_x = r_{xx} = 1$.
\end{proof}

\subsection{Proof for Proposition~\ref{prop:improved-utility}}\label{proof:improved-utility}
Let $f([Z_n, Z_u] / n) = \frac{u + Z_u}{n + Z_n}$.
As in the proof of Proposition~\ref{prop:naive-utility}, it suffices to show
1) $\Prob{Z \notin A} = o(1/n^2)$,
and 2) $\Expec{\norm{\Delta(Z)}^2} = \frac{17/2}{\varepsilon^2}$.
The proof of 1) can be done in a similar way as Lemma~\ref{lem:probability-bound}. Thus, we prove 2) in the following lemma:
\begin{lemma}
    We have
    \begin{align*}
        \Expec{\norm{\Delta(Z)}^2} & \leq \frac{17/2}{\varepsilon^2}.
    \end{align*}
\end{lemma}
\begin{proof}
    From simple algebra, we have
    \begin{align*}
        \Delta(Z) = -u\cdot Z_n + Z_u
    \end{align*}
    Since $Z \sim \Lap(2/\varepsilon)$, we obtain
    \begin{align*}
        \Expec{\norm{\Delta(Z)}^2} & = 8/\varepsilon^2 \ab(u^2 + 1)  \\
                                   & \leq \frac{17/2}{\varepsilon^2}
    \end{align*}
    Here, the equality holds when $u = 1/4$.
\end{proof}

\subsection{Proof of Lemma~\ref{lem:add-remove-model-sensitivity}} \label{proof:add-remove-model-sensitivity}
\subsubsection{Variance Case}
\begin{proof}
    Let $a \in \R^n, a' \in \R^{n+1}$ be neighboring datasets, i.e., $a_i = a'_i$ for all $i \leq n$ and $a'_{n+1} = x$.
    To simplify the notation, we denote $r = \frac{1}{n} \sum_{i=1}^{n}a_i$
    Then, we have
    \begin{align*}
        \uvar(a') - \uvar(a)
         & = \sum_{i=1}^{n+1} a'_i - \frac{1}{n+1}\left(\sum_{i=1}^{n+1} a'_i\right)^2 - \sum_{i=1}^n a_i + \frac{1}{n}\left(\sum_{i=1}^n a_i\right)^2 \\
         & = x^2 - \frac{1}{n+1}\left(nr + x\right)^2 + \frac{1}{n}\left(nr\right)^2                                                                   \\
         & = x^2 - \frac{1}{n+1}\left(nr + x\right)^2 + \frac{1}{n+1}n(n+1)r^2                                                                         \\
         & = x^2 + \frac{1}{n+1}[n(n+1)r^2 -\left(nr + x\right)^2]                                                                                     \\
         & = x^2 + \frac{1}{n+1}[n(n+1)r^2 -n^2r^2 - 2nrx - x^2]                                                                                       \\
         & = \frac{1}{n+1}[nr^2 - 2nrx + nx^2]                                                                                                         \\
         & = \frac{n}{n+1}(r - x)^2 .
    \end{align*}
    Thus, we have
    \begin{align*}
        0 \leq \uvar(a') - \uvar(a) \leq 1,
    \end{align*}
    which completes the proof.
\end{proof}

\subsubsection{Covariance Case}
\begin{proof}
    Let $(a, b) \in \R^n \times \R^n, (a', b') \in (\R^{n+1} \times \R^{n+1})$ be neighboring datasets, i.e.,
    $a_i = a'_i$ for all $i \leq n, a'_{n+1} = x$ and $b_i = b'_i$ for all $i \leq n, b'_{n+1} = y$.
    To simplify the notation, we denote $r_a = \frac{1}{n} \sum_{i=1}^{n}a_i$ and $r_b = \frac{1}{n} \sum_{i=1}^{n}b_i$.
    Then, we have
    \begin{align*}
        \ucov(a', b') - \ucov(a, b)
         & = \sum_{i=1}^{n+1} a'_i b'_i - \frac{1}{n+1}\left(\sum_{i=1}^{n+1} a'_i\right)\left(\sum_{i=1}^{n+1} b'_i\right) - \sum_{i=1}^n a_i b_i + \frac{1}{n}\left(\sum_{i=1}^n a_i\right)\left(\sum_{i=1}^n b_i\right) \\
         & = xy - \frac{1}{n+1}\left(nr_a + x\right)\left(nr_b + y\right) + \frac{1}{n}\left(nr_a\right)\left(nr_b\right)                                                                                                  \\
         & = xy - \frac{1}{n+1}\left(n^2r_ar_b + nr_ay + nr_bx + xy\right) + \frac{1}{n+1}n(n+1)r_ar_b                                                                                                                     \\
         & = - \frac{1}{n+1}\left(-nr_ar_b + nr_ay + nr_bx - nxy\right)                                                                                                                                                    \\
         & = \frac{n}{n+1}(r_a - x)(r_b - y)                                                                                                                                                                               \\
    \end{align*}
    Thus, we have
    \begin{align*}
        -1 & \leq \ucov(a', b') - \ucov(a, b) \leq 1,
    \end{align*}
    which completes the proof.
\end{proof}

\subsection{Proof of Theorem~\ref{thm:bezier-utility}}\label{app:proof-bezier-utility}
From the definition of $\hat \mu_j$, we have
\begin{align*}
    \hat \mu_j & = [A^{-1} (A \mu + \xi)]_j \\
               & = \mu_j + [A^{-1} \xi]_j.
\end{align*}
Since $\xi_j \sim \Lap(1/\varepsilon)$, we arrive at
\begin{align*}
    E[(\hat \mu_j - \mu_j)^2] & = \frac{2}{\varepsilon^2} \sum_{l=0}^k (A^{-1}_{jl})^2                           \\
                              & = \frac{2}{\varepsilon^2} \sum_{l=j}^k \ab(\frac{\binom{l}{j}}{\binom{k}{j}})^2,
\end{align*}
where, we used the fact that $A^{-1}_{jl} = \frac{\binom{l}{j}}{\binom{k}{j}}$~\citep{simsek2013functional}.
For $j = k$, we have
\begin{align}
    E[(\hat \mu_k - \mu_k)^2] & = \frac{2}{\varepsilon^2}. \label{eq:mu-k-utility}
\end{align}

For any $\varepsilon$-DP estimator $\tilde \mu_k$,
we can construct an $\varepsilon$-DP estimator $\hat s$ of $\mu_1$ as follows:
\begin{align*}
    \hat s(D) & = \tilde \mu_k(\tilde D),
\end{align*}
where $D = [x]$ and $\tilde D = [[x_1^{1/k}, \dots, x_n^{1/k}]^\top]$.
Note that $\mu_1(D) = \mu_k(\tilde D)$.
The utility of $\hat s$ can be evaluated as
\begin{align*}
    E[(\hat s(D) - \mu_1(D))^2] & = E[(\tilde \mu_k(\tilde D) - \mu_k(\tilde D))^2] \geq \sigma(\varepsilon)^2 (1 - o(1))
\end{align*}
from Lemma 6 in~\citet{kulesza2024mean}.
In the high-privacy regime, the lower bound becomes $2/\varepsilon^2(1 - o(1))$.
Thus, Eq.~\ref{eq:mu-k-utility} is asymptotically optimal for estimation of $\mu_k$.

\subsection{Proof of Lemma~\ref{lem:multi-dim-sensitivity}}\label{proof:multi-dim-sensitivity}

\begin{proof}
    For neighboring datasets $D$ and $D'$, we have
    \begin{align*}
        \norm{b(D') - b(D)}_1 & = \sum_{\alpha \in \{0, \dots, k\}^d} \abs{B_{\alpha_1}^{k}(x^{(1)}_{n+1})\cdots B_{\alpha_d}^{k}(x^{(d)}_{n+1})}          \\
                              & = \sum_{\alpha \in \{0, \dots, k\}^d} B_{\alpha_1}^{k}(x^{(1)}_{n+1})\cdots B_{\alpha_d}^{k}(x^{(d)}_{n+1})                \\
                              & = \ab(\sum_{\alpha_1=0}^k B_{\alpha_1}^{k}(x^{(1)}_{n+1})) \cdots \ab(\sum_{\alpha_d=0}^k B_{\alpha_d}^{k}(x^{(d)}_{n+1})) \\
                              & = 1,
    \end{align*}
    where the first equality follows from the non-negativity of Bernstein basis and the last equality follows from the partition of unity property.
\end{proof}

\subsection{Proof of Theorem~\ref{thm:covariance-estimation-add-remove-utility}} \label{app:proof-covariance-estimation-add-remove-utility}
Let $Z = [Z_n, Z_x, Z_y, Z_{xy}]$ and $f(Z / n) = \frac{r_{xy} + Z_{xy} / n}{1 + Z_n / n} - (\frac{r_x + Z_x / n}{1 + Z_n / n})(\frac{r_y + Z_y / n}{1 + Z_n / n})$.
Note that $f(0)$ is the true covariance and $\hat c = \clip(f(Z/n), [-1/4, 1/4])$, where
$Z_n = \sum_{i} Z_i, Z_x = Z_{1} + Z_{3}, Z_y = Z_{2} + Z_{3}, Z_{xy} = Z_{3}$
and $Z_i\sim \Lap(1/\varepsilon)$.
From simple calculation, we have $\grad f(Z/n), \grad^2 f(Z/n) = O(1)$ for $Z \in A_n$.
Thus, from Lemma~\ref{lem:proxy-bound}, it suffices to show
1) $\Prob{A_n} = 1 - o(1/n^2)$,
and 2) $\Expec{\norm{\Delta(Z)}^2} \leq 2 / \varepsilon^2$.
We prove these conditions in the following lemmas:
\begin{lemma}\label{lem:laplace-clip-probability}
    We have
    \begin{align*}
        \Prob{Z \in A_n} & = 1 - O(\exp(-n^{1/4} / (4b))) = 1 - o(1/n^2).
    \end{align*}
\end{lemma}
\begin{proof}
    Since $\abs{Z_i} \leq 1/4 n^{1/4}$ for all $i$
    is sufficient for the event $A_n$ to hold, we can bound the probability of $Z \in A_n$ as follows:
    \begin{align*}
        \Prob{Z \in A_n} & \leq 1 - \sum_{i=1}^4 \Prob{\abs{Z_i} > 1/4 n^{1/4}} \\
                         & \leq 1 - 4\exp\ab(-\frac{n^{1/4}}{4b})
    \end{align*}
    This completes the proof.
\end{proof}

\begin{lemma}\label{lem:first-order-variance}
    Let $r_x = \frac{1}{n} \sum_{i=1}^{n} x_i$, $r_y = \frac{1}{n} \sum_{i=1}^{n} y_i$, and $c = \frac{1}{n}\sum_{i=1}^n (x_i - r_x)(y_i - r_y)$.
    Then, we have
    \begin{align*}
        \Expec{\Delta(Z)^2} = \frac{2}{\varepsilon^2} C(r_x, r_y, c) \leq \frac{2}{\varepsilon^2},
    \end{align*}
    where $C(r_x, r_y, c) := (1 - 2r_x + 2r_x^2)(1-2r_y + 2r_y^2) - 2c(1-2r_x)(1-2r_y) + 4c^2$.
\end{lemma}
\begin{proof}
    Let $r_{xy} = \frac{1}{n} \sum_{i=1}^n x_iy_i$.
    From simple calculation, we have
    \begin{align*}
        \Delta(Z) = Z_{xy} - r_{xy}Z_n + 2r_xr_yZ_n - r_y Z_x - r_x Z_y
    \end{align*}
    To simplify the expression, we rewrite $\Delta(Z)$ as follows:
    \begin{align*}
        \Delta(Z) = \underbrace{(Z_{xy} - r_y Z_x - r_x Z_y + r_x r_y Z_n)}_{=:X} - c Z_n.
    \end{align*}
    Substituting the definition of $Z_x, Z_y$ and $Z_{xy}$ into the expression for $X$, we have $X = \sum_{i=1}^4 x_i Z_i$, where
    \begin{align*}
        x_1 & = (1-r_x)(1-r_y) \\
        x_2 & = -r_y(1-r_x)    \\
        x_3 & = -r_x(1-r_y)    \\
        x_4 & = r_x r_y.
    \end{align*}
    In addition, $Y$ can be expressed as $Y = \sum_{i=1}^4 (x_i - c) Z_i$ since $Z_n = Z_1 + Z_2 + Z_3 + Z_4$.
    Thus, we have
    \begin{align*}
        \frac{\Expec{Y^2}}{2/\varepsilon^2} = \sum_{i=1}^4 (x_i - c)^2 = \sum_{i=1}^4 x_i^2 - 2c \sum_{i=1}^4 x_i + 4c^2.
    \end{align*}
    Here, $\sum_{i=1}^4 x_i^2$ and $\sum_{i=1}^4 x_i$ can be computed as follows:
    \begin{align*}
        \sum_{i=1}^4 x_i^2 & = (1-r_x)^2(1-r_y)^2 + r_y^2(1-r_x)^2 + r_x^2(1-r_y)^2 + r_x^2r_y^2 \\
                           & = \left((1-r_x)^2 + r_x^2\right)\left((1-r_y)^2 + r_y^2\right)      \\
                           & = (1-2r_x+2r_x^2)(1-2r_y+2r_y^2)                                    \\
        \sum_{i=1}^4 x_i   & = (1-r_x)(1-r_y) - r_y(1-r_x) - r_x(1-r_y) + r_xr_y                 \\
                           & = (1-2r_x)(1-2r_y).
    \end{align*}
    Therefore, it suffices to show that
    \begin{align*}
        f(c) & := C_1(r_x, r_y, c) := \frac{\Expec{Y^2}}{2/\varepsilon^2}          \\
             & = (1-2r_x+2r_x^2)(1-2r_y+2r_y^2) - 2c(1-2r_x)(1-2r_y) + 4c^2 \leq 1
    \end{align*}
    Note that $f(c)$ is a convex quadratic function of $c$.
    The maximum value of a convex function on a closed interval is achieved at one of the endpoints of the interval. Therefore, it suffices to show that $f(c) \leq 1$ at the endpoints of the interval of $c$.
    As shown in Lemma~\ref{lem:rxy}, the domain of $r_{xy}$ is given by $\max(0, r_x + r_y - 1) \leq r_{xy} \leq \min(r_x, r_y)$, which defines the interval for $c = r_{xy} - r_x r_y$.

    \textbf{Case 1: $r_{xy} = \min(r_x, r_y)$.}
    Without loss of generality, assume $r_x \leq r_y$, which implies $r_{xy} = r_x$.
    Then, $c = r_x - r_x r_y = r_x(1-r_y)$.
    Substituting this into $f(c)$ yields
    \begin{align*}
        g(r_x, r_y) := f(r_x - r_xr_y) - 1 & = 16r_x^2r_y^2 - 24r_x^2r_y + 10r_x^2 - 8r_xy^2 + 10r_xy - 4r_x + 2r_y^2 - 2r_y \\
                                           & = (16r_x^2 - 8r_x + 2)r_y^2 + (-24r_x^2 + 10r_x - 2)r_y + (10r_x^2 - 4r_x)
    \end{align*}
    The function $g(r_x, r_y)$ is a convex quadratic function in $r_y$
    since $(16r_x^2 - 8r_x + 2) \geq 0$ for any $r_x$.
    Thus, the maximum value of $g(r_x, r_y)$ occurs at one of the endpoints $r_y = r_x$ or $1$.
    For $r_y = r_x$, we have
    \begin{align*}
        g(r_x, r_x) & = 16r_x^4 - 24r_x^3 + 10r_x^2 - 8r_x^3 + 10r_x^2 - 4r_x + 2r_x^2 - 2r_x \\
                    & = 16r_x^4 - 32r_x^3 + 22r_x^2 - 6r_x                                    \\
                    & = 2r_x(8r_x^3 - 16r_x^2 + 11r_x - 3)                                    \\
                    & = 2r_x(r_x-1)(8r_x^2 - 8r_x + 3)                                        \\
                    & \leq 0
    \end{align*}
    since $(8r_x^2 - 8r_x + 3) > 0$, $r_x \geq 0$, and $(1-r_x) \leq 0$.
    For $r_y = 1$, we have
    \begin{align*}
        g(r_x, 1) & = (16r_x^2 - 8r_x + 2) - (24r_x^2 - 10r_x + 2) + (10r_x^2 - 4r_x) \\
                  & = (16 - 24 + 10)r_x^2 + (-8 + 10 - 4)r_x + (2 - 2)                \\
                  & = 2r_x^2 - 2r_x                                                   \\
                  & = 2r_x(r_x - 1)                                                   \\
                  & \leq 0.
    \end{align*}

    \textbf{Case 2: $r_{xy} = \max(0, r_x + r_y - 1)$.}
    Assuming $r_x+r_y \leq 1$, we have $c = -r_xr_y$.
    Then, $f(c)$ can be expressed as follows:
    \begin{align*}
        g(r_x, r_y) & := f(-r_xr_y) - 1                                                                   \\
                    & = - 2r_x - 2r_y + 2r_x^2 + 2r_y^2 + 6r_xr_y - 8r_x^2r_y - 8r_xr_y^2 + 16r_x^2r_y^2.
    \end{align*}
    Here, $g(r_x, r_y)$ is a convex quadratic function in $r_x$ since the coefficient of $r_x^2$ is $16r_y^2 - 8r_y + 2 \geq 0$ for any $r_y \in [0, 1]$.
    Since $r_x \in [0, 1-r_y]$, the maximum value occurs at one of the endpoints $r_x = 0$ or $r_x = 1-r_y$.
    For $r_x = 0$, we have
    \begin{align*}
        g(0, r_y) & = -2r_y + 2r_y^2 \\
                  & = 2r_y(r_y - 1)  \\
                  & \leq 0.
    \end{align*}
    For $r_x = 1-r_y$, let $p = r_xr_y$.
    Then, we have
    \begin{align*}
        g(1-r_y, r_y) & = -2 + 2 + 2p - 8p + 16p^2 \\
                      & = 16p^2 - 6p               \\
                      & = 16(p - 3/16)^2 - 9/16    \\
                      & \leq 0
    \end{align*}
    since $p = r_x(1-r_x) \in [0, 1/4]$.

    Assuming $r_x+r_y \geq 1$, we have $c = r_x + r_y - 1 - r_xr_y = -(1-r_x)(1-r_y)$.
    Then, we can express $f(c)$ as follows:
    \begin{align*}
        g(r_x, r_y) & := f(-(1-r_x)(1-r_y)) - 1                                                                 \\
                    & =16r_x^2r_y^2 - 24r_x^2r_y + 10r_x^2 - 24r_xr_y^2 + 38r_xr_y - 16r_x+10r_y^2 - 16r_y + 6.
    \end{align*}
    Here, $g(r_x, r_y)$ is a convex quadratic function in $r_x$ since the coefficient of $r_x^2$ is $16r_y^2 - 24r_y + 10 \geq 0$ for any $r_y \in [0, 1]$.
    Since $r_x \in [1-r_y, 1]$, the maximum value of $g(r_x, r_y)$ occurs at one of the endpoints $r_x = 1-r_y$ or $1$.
    For $r_x = 1$, we have
    \begin{align*}
        g(1, r_y) & = 16r_y^2 - 24r_y + 10 - 24r_y^2 + 38r_y - 16 + 10r_y^2 - 16r_y + 6 \\
                  & = 2r_y^2 - 2r_y \leq 0.
    \end{align*}
    For $r_x = 1 - r_y$, let $p = r_xr_y$.
    Then, we have
    \begin{align*}
        g(1 - r_y, r_y) & = 16p^2 - 24p + 10 + 18p - 10 \\
                        & = 16p^2 - 6p                  \\
                        & = 16(p - 3/16)^2 - 9/16       \\
                        & \leq 0.
    \end{align*}
    For the last inequality, we use the fact that $p = r_xr_y \leq (1-r_y)r_y \in [0, 1/4]$.

    Combining Case 1 and Case 2, we obtain the result.
\end{proof}

\subsection{Proof of Theorem~\ref{thm:variance-estimation-add-remove-utility}}~\label{app:proof-variance-estimation-add-remove-utility}
The proof follows the same structure as the proof of Theorem~\ref{thm:covariance-estimation-add-remove-utility}.
Let $Z = [Z_n, Z_x, Z_{xx}]$ and $f(Z / n) = \frac{r_{xx} + Z_{xx}/n}{1 + Z_n / n} - (\frac{r_x + Z_x / n}{1 + Z_n / n})^2$.
Note that $f(0)$ is the true variance and $\hat v = \clip(f(Z / n), [0, 1/4])$, where $Z_n = Z_0 + Z_1 + Z_2$, $Z_x = Z_1 / 2 + Z_2$, and $Z_{xx} = Z_2$ for $Z_i \sim \Lap(1/\varepsilon)$.
From Lemma~\ref{lem:proxy-bound}, it suffices to show
1) $\Prob{A_n} = 1 - o(1/n^2)$,
and 2) $\Expec{\norm{\Delta(Z)}^2} \leq 2 / \varepsilon^2$.
\begin{lemma}
    The probability of the event $A_n$ is bounded as follows:
    \begin{align*}
        \Prob{Z \in A_n} & = 1 - O(\exp(-n^{1/4} / (4b))).
    \end{align*}
\end{lemma}
\begin{proof}
    Since $\abs{Z_i} \leq 1/4 n^{1/4}$ for all $i$
    is sufficient for the event $A_n$ to hold, we can bound the probability of $A_n$ as follows:
    \begin{align*}
        \Prob{Z \in A_n} & \leq 1 - \sum_{i=1}^3 \Prob{\abs{Z_i} > 1/4 n^{1/4}}   \\
                         & \leq 1 - 4\exp\ab(-\frac{n^{1/4}}{4b}) = 1 - o(1/n^2).
    \end{align*}
    This completes the proof.
\end{proof}

\begin{lemma}\label{lem:first-order-variance-3}
    Let $r = \frac{1}{n} \sum_{i=1}^{n} x_i$ and $v = \frac{1}{n} \sum_{i=1}^n (x_i - r)^2$.
    Then, we have
    \begin{align*}
        \Expec{\Delta(Z)^2} = \frac{2}{\varepsilon^2} C_b(r, v) \leq \frac{2}{\varepsilon^2},
    \end{align*}
    where $C_b(r, v) := 3v^2 - 2(3r^2 - 3r + 1)v + 3r^4 - 6r^3 + 7r^2 - 4r + 1$.
\end{lemma}
\begin{proof}
    Let $s = \frac{1}{n} \sum_{i=1}^n x_i^2$.
    From simple calculation, we have
    \begin{align*}
        \Delta(Z) = Z_{xx} - sZ_n + 2r^2Z_n - 2r Z_x.
    \end{align*}
    Let $Y = Z_{xx} - sZ_n + 2r^2Z_n - 2r Z_x$.
    Then, we have
    \begin{align*}
        Y & = (2r^2 - s)Z_1 + (2r^2 - r - s)Z_2 + (1 - 2r + 2r^2 - s)Z_3.
    \end{align*}
    Since the $Z_i$ are independent and have variance $2 / \varepsilon^2$, $\Expec{Y^2 / (2/\varepsilon^2)}$ is given by:
    \begin{align*}
        \Expec{Y^2 / (2/\varepsilon^2)} & = (2r^2 - s)^2 + (2r^2 - r - s)^2 + (1 - 2r + 2r^2 - s)^2 \\
                                        & = (r^2 - v)^2 + (r^2 - r - v)^2 + ((1 - r)^2 - v)^2       \\
                                        & = 3v^2 - 2(3r^2 - 3r + 1)v + 3r^4 - 6r^3 + 7r^2 - 4r + 1  \\
                                        & =: C_3(r, v).
    \end{align*}
    We compare $C_3$ with $C_1(r, r, v)$ in the proof of Theorem~\ref{thm:covariance-estimation-add-remove-utility}.
    We have
    \begin{align*}
        C_1(r, r, v) - C_3(r, v) & = v^2 - 2(r^2 - r)v + r^2(r - 1)^2 = (v - (r^2 - r))^2 \geq 0.
    \end{align*}
    Thus, we have
    \begin{align*}
        C_3(r, v) \leq C_1(r, r, v) \leq 1.
    \end{align*}
    This completes the proof.
\end{proof}

\subsection{Proof of Lemma~\ref{lem:transformed-sensitivity}} \label{app:proof-transformed-sensitivity}
\begin{proof}
    Let $D = [x]$ and $D' = [x']$ be neighboring datasets, i.e., $x_i = x'_i$ for all $i \leq n$ and $x'_{n+1} = a$.
    Then, we have
    \begin{align*}
        \norm{b(D') - b(D)}_1 = \abs{1 - (\uvar(x') - \uvar(x))} + \abs{\uvar(x') - \uvar(x)} \leq 1.
    \end{align*}
    Here, we used the fact $0 \leq \ucov(x') - \ucov(x) \leq 1$ from Lemma~\ref{lem:add-remove-model-sensitivity}.
\end{proof}

\subsection{Proof of Proposition~\ref{prop:variance-estimation-transformed-utility}} \label{app:proof-variance-estimation-transformed-utility}
The proof can be done in a similar way to the proof of Theorem~\ref{thm:covariance-estimation-add-remove-utility}.
Thus, it suffices to show Lemmas corresponding to Lemma~\ref{lem:laplace-clip-probability}, and~\ref{lem:first-order-variance}.
\begin{lemma}
    The probability of the event $A$ defined above is given by
    \begin{align*}
        \Prob{A} & = 1 - O(\exp(-n^{1/4} / (4b))).
    \end{align*}
\end{lemma}
\begin{proof}
    Since $\abs{Z_i} \leq 1/4 n^{1/4}$ for all $i$
    is sufficient for the event $A$ to hold, we can bound the probability of $A$ as follows:
    \begin{align*}
        \Prob{A} & \leq 1 - \sum_{i=1}^4 \Prob{\abs{Z_i} > 1/4 n^{1/4}} \\
                 & \leq 1 - 4\exp\ab(-\frac{n^{1/4}}{4b})
    \end{align*}
    This completes the proof.
\end{proof}
\begin{lemma}
    Let $r = \frac{1}{n}\sum_{i=1}^n x_i$ and $v = \frac{1}{n} \sum_{i=1}^n (x_i - r)^2$.
    Then, we have
    \begin{align*}
        \Expec{\Delta(Z)^2} = \frac{2}{\varepsilon^2}C_u(r, v) \leq \frac{2}{\varepsilon^2},
    \end{align*}
    where $C_u(r, v) = v^2 + (1 - v)^2$.
\end{lemma}
\begin{proof}
    From simple calculation, we have
    \begin{align*}
        \Delta(Z) & = Z_u - vZ_n = -v Z_1 + (1 - v)Z_2.
    \end{align*}
    Since the $Z_i$ are independent and have variance $2 / \varepsilon^2$,
    we obtain
    \begin{align*}
        \Expec{\Delta(Z)^2 / (2/\varepsilon^2)} & = C_u(r, v) \leq 1
    \end{align*}
    since $v \in [0, 1/4]$.
    This completes the proof.
\end{proof}

\subsection{Proof of Theorem~\ref{thm:comparison}}\label{proof:comparison}
The upper bound on $R(\cdot, \vvar, D)$ is already obtained in the proof of Theorem~\ref{thm:covariance-estimation-add-remove-utility},~\ref{thm:variance-estimation-add-remove-utility}, and Proposition~\ref{prop:variance-estimation-transformed-utility}.
In addition, we have already show that $C_b(r, v) \leq C_c(r, v) = C(r, r, v)$ in the proof of Lemma~\ref{lem:first-order-variance-3}.
Thus, it suffices to show $C_c(r, v) \leq C_u(r, v)$.
Let $f(r, v) = (C_c(r, v) - C_u(r, v)) / 2$.
Then, we have
\begin{align*}
    f(r, v) & = v^2 - (4r^2 - 4r)v + (2r^4 - 4r^3 + 4r^2 - 2r).
\end{align*}
This is a convex quadratic function in $v$.
Thus, the maximum value of $f(r, v)$ occurs at one of the endpoints $v = 0$ or $r - r^2$.
For $v = 0$, we have
\begin{align*}
    f(r, 0) & = 2r^4 - 4r^3 + 4r^2 - 2r     \\
            & = 2r(r^3 - 2r^2+2r - 1)       \\
            & = 2r(r-1)(r^2 - r + 1) \leq 0
\end{align*}
for any $r \in [0, 1]$.
For $v = r - r^2$, we have
\begin{align*}
    f(r, r - r^2) & = (r - r^2)^2 - (4r^2 - 4r)(r - r^2) + (2r^4 - 4r^3 + 4r^2 - 2r) \\
                  & = 5(r - r^2)^2 + (r - r^2)(-2 + 2r - 2r^2)                       \\
                  & = (r - r^2)(-7r^2 + 7r - 2)                                      \\
                  & = -r(1 - r)(7(r - 1/2)^2 + 1/4) \leq 0.
\end{align*}
Combining the above two cases, we have $f(r, v) \leq 0$ for any $r \in [0, 1]$ and $v \in [0, r - r^2]$.
Thus, we obtain the result.